\documentclass{article}

\usepackage[utf8]{inputenc} % allow utf-8 input
\usepackage[T1]{fontenc}    % use 8-bit T1 fonts

\usepackage[a4paper, total={6.25in, 9in}]{geometry}

\usepackage{hyperref}       % hyperlinks
\usepackage{url}            % simple URL typesetting
\usepackage{booktabs}       % professional-quality tables
\usepackage{amsfonts}       % blackboard math symbols
\usepackage{nicefrac}       % compact symbols for 1/2, etc.
\usepackage{microtype}      % microtypography
\usepackage[round]{natbib}
\usepackage{amsmath,amsthm,amssymb,amsfonts}
\usepackage{algorithm,algorithmic}
\usepackage[dvipsnames,table]{xcolor}
\usepackage{xspace}
\usepackage{dsfont}
\usepackage{makecell}
\usepackage{nicefrac}

\definecolor{pearThree}{HTML}{E74C3C}
\definecolor{pearcomp}{HTML}{B97E29}
\definecolor{pearDark}{HTML}{2980B9}
\definecolor{pearDarker}{HTML}{1D2DEC}
\hypersetup{
	colorlinks,
	citecolor=pearDark,
	linkcolor=pearThree,
	breaklinks=true,
	urlcolor=black}

\newtheorem{definition}{Definition}

\newtheorem{proposition}{Proposition}

\newtheorem{theorem}{Theorem}
\newtheorem{lemma}{Lemma}
\newtheorem{remark}{Remark}

\usepackage{macros}

\usepackage{authblk}

\title{Towards Instance-Optimality in Online \\ PAC Reinforcement Learning}

\author[1]{Aymen Al-Marjani}
\author[ ]{Andrea Tirinzoni}
\author[2]{Emilie Kaufmann}
\affil[1]{UMPA, ENS Lyon, Lyon, France}
\affil[2]{Univ. Lille, CNRS, Inria, Centrale Lille, UMR 9189 - CRIStAL, Lille, France}

\date{}
\usepackage{minitoc}

\begin{document}

\maketitle

% For TOC in appendix (https://tex.stackexchange.com/a/419290)
\doparttoc % Tell to minitoc to generate a toc for the parts
\faketableofcontents % Run a fake tableofcontents command for the partocs

\begin{abstract}
Several recent works have proposed instance-dependent upper bounds on the number of episodes needed to identify, with probability $1-\delta$, an $\varepsilon$-optimal policy in finite-horizon tabular Markov Decision Processes (MDPs). These upper bounds feature various complexity measures for the MDP, which are defined based on different notions of sub-optimality gaps. However, as of now, no {lower bound} 
%on the sample complexity 
has been established to assess the optimality of any of these complexity measures, except for the special case of MDPs with deterministic transitions. In this paper, we propose the first instance-dependent lower bound on the sample complexity required for the PAC identification of a near-optimal policy in any tabular episodic MDP. Additionally, we demonstrate that the sample complexity of the PEDEL algorithm of \cite{Wagenmaker22linearMDP} closely approaches this lower bound. Considering the intractability of PEDEL, we formulate an open question regarding the possibility of achieving our lower bound using a computationally-efficient algorithm.
\end{abstract}

\section{Introduction}
We consider the online Probably Approximately Correct Reinforcement Learning (PAC RL) problem, in which an agent sequentially interacts with an environment modeled as a Markov Decision Process (MDP), with the goal of learning a near-optimal policy as quickly as possible.
More precisely, given a precision $\epsilon \ge 0$ and a risk parameter $\delta \in (0,1)$, the agent is required to return a policy $\widehat{\pi}$ whose value is within $\epsilon$ of the value of the optimal policy, with probability at least $1-\delta$. The agent's performance is evaluated through its \emph{sample complexity}, defined as the number of interactions with the environment needed to output such a policy $\widehat{\pi}$. 

%This problem has been mostly studied in two different settings: in discounted MDPs, in which the value of a policy is the expected (infinite) sum of rewards discounted by a factor $\gamma \in (0,1)$, and in finite-horizon (or episodic) MDPs, in which the value is the expected sum of rewards up to a given horizon $H$. In both cases, the performance of an algorithm is evaluated through its \emph{sample complexity}, defined as the number of interactions with the environment (or the number of episodes, in the episodic setting) needed to output such a policy $\widehat{\pi}$. 

%Since its introduction by \cite{Fiechter94}, this PAC RL problem has been extensively investigated from a minimax point of view for both episodic and infinite-horizon discounted MDPs, with an emphasis on the complexity of learning with a generative model \citep{azar13GeneJournal,sidford18,agarwal20gene}, and in the more realistic setting where interaction is allowed only through trajectories \citep{dann15PAC,Dann2019certificates,Kaufmann21RFE,Menard21RFE}.

Since its introduction by \cite{Fiechter94}, this problem has been extensively investigated from a \emph{minimax} point of view in two different settings: discounted MDPs \citep{azar13GeneJournal,sidford18,agarwal20gene}, in which the value of a policy is the expected (infinite) sum of rewards discounted by a factor $\gamma \in (0,1)$, and finite-horizon (or episodic) MDPs \citep{dann15PAC,Dann2019certificates,Kaufmann21RFE,Menard21RFE}, in which the value is the expected sum of rewards up to a given horizon $H$. Notably, in the finite-horizon setting with $S$ states, $A$ actions, and horizon $H$, \cite{dann15PAC} proved that any PAC RL agent must play at least $\Omega(SAH^2 \log(1/\delta)/\epsilon^2)$ episodes to identify an $\epsilon$-optimal policy in the \emph{worst-case}. Their lower bound was derived under the assumption of time-homogeneous rewards and transitions, while a lower bound of $\Omega(SAH^3 \log(1/\delta)/\epsilon^2)$ episodes was later derived by \cite{domingues2021episodic} for the time-inhomogeneous case. There exist algorithms with sample complexity matching these lower bounds \citep{Dann2019certificates,Menard21RFE}.

Unfortunately, minimax optimality is not informative about the performance of an algorithm under different MDPs of the same size $(H,S,A)$. For instance, let us imagine a first MDP with deterministic transitions and a tree structure which has a single optimal trajectory whose rewards are all considerably higher than the rewards in any other trajectory. Let us also consider a second MDP in which all actions yield exactly the same reward, but this information is unknown to the agent beforehand. One would naturally expect the PAC RL task to be much easier in the first MDP where a few episodes should suffice to detect that the policy following the good trajectory is optimal.
%``stands from the crowd''. 
In the second one, however, no reasonable algorithm can \emph{confidently state} that a policy is $\epsilon$-optimal before having estimated uniformly well (with $\epsilon$-precision) the value of all other policies.
%stop before it has gathered enough observations from all regions of the MDP to estimate all policies with $\epsilon$-precision. 

%This motivates a recent line of works which try to go beyond minimax optimality by proposing adaptive algorithms that enjoy \emph{instance-dependent} guarantees. 
This motivates a recent line of works focused on designing adaptive algorithms with \emph{instance-dependent} guarantees, {i.e., sample complexity bounds featuring some characteristics of the underlying MDP that go beyond its size as in minimax results. These characteristics have been expressed with different notions of \emph{sub-optimality gaps}}. The first algorithm of this kind is BESPOKE \citep{zanette2019almost}, which was proposed for discounted MDPs. Its gap-based sample complexity is shown to be never worse than the minimax rate, while it can be significantly smaller in some MDPs. \cite{Taupin2022BestPI} later proposed GSS, a PAC RL algorithm for discounted linear MDPs \citep{Jin2019ProvablyER} along with GSS-E, its counterpart for episodic linear MDPs. The problem of exact identification of the optimal policy ($\epsilon=0$) and the more complex identification of a Blackwell-optimal policy were treated by \cite{marjani2021adaptive} and \cite{boone23a}, respectively. All these works assume that a generative model is available, i.e., that the learner can query a transition from any state at any time. In the more challenging setting where interaction is allowed only through trajectories, \cite{al2021navigating} studied exact best-policy identification in discounted MDPs, while a more recent line of works has considered approximate identification ($\varepsilon \geq 0$) in episodic MDPs  \citep{wagenmaker21IDPAC,tirinzoni2022NearIP,Wagenmaker22linearMDP,tirinzoni23optimistic,almarjani23covgame}. All the proposed algorithms have sample complexity upper bounds of the form $\widetilde{\cO}\big(\cC(\cM, \epsilon) \log(1/\delta)\big)$, where $\cC$ quantifies the hardness of learning an $\epsilon$-optimal policy in the MDP $\cM$, $\delta$ is the risk, while $\widetilde{\cO}$ hides numerical constants and logarithmic factors of the relevant parameters. The expression of $\cC$ is different for each algorithm but always dependent on some sub-optimality gaps (either values gaps or policy gaps, see Section \ref{sec:prelim} for a formal definition) and on state-visitation probabilities. We review these bounds in Section \ref{sec:upper_bounds}.

However, the lack of a general instance-dependent lower bound makes it difficult to assess the optimality of these approaches, i.e., how tight a complexity $\cC(\cM, \epsilon)$ is compared to the best possible rate. Indeed, the only instance-dependent lower bounds for PAC RL without a generative model are either restricted to MDPs with deterministic transitions \citep{tirinzoni2022NearIP} or cover only the case of exact best-policy identification ($\epsilon=0$) under the assumption that the optimal policy is unique \citep{al2021navigating}. In this work, we fill this gap by answering the following question:
% \begin{center}
%     \emph{Given an episodic tabular MDP $\cM$, what is the best possible rate in $\log(1/\delta)$ that a PAC RL algorithm can achieve ?}
% \end{center}
\begin{center}
    \emph{What is the best rate in $\log(1/\delta)$ that a PAC RL algorithm can achieve on an episodic tabular MDP?}
\end{center}

\paragraph{Contributions} We derive the first instance-dependent lower bound for PAC RL that holds for any $\epsilon \geq 0$ and any tabular MDP (Theorem \ref{thm:general-lb-epsilon}). As for bandit identification problems with many correct answers \citep{DegenneK19}, our lower bound holds when $\delta \to 0$. Beyond the asymptotic regime, we strengthen this result with an additional lower bound that holds for all $\delta >0$ in the special case of $\epsilon =0$ under the assumption that optimal policies share a unique state-action distribution (Theorem \ref{thm:lb-exact-identification}). Then, in Section \ref{sec:upper_bounds}, we review the complexity measures featured in existing upper bounds and show that the PEDEL algorithm of \cite{Wagenmaker22linearMDP} matches our lower bound in tabular MDPs up to multiplicative $H$ factors and an additive $\widetilde{\cO}({1}/{\epsilon^2})$ term (Proposition \ref{prop:PEDEL-lb}). 
A shortcoming of PEDEL is that it is not computationally efficient as it explicitly enumerates all policies. We thus formulate an open question as to whether our bound can be attained by a computationally-efficient algorithm.  

\section{Preliminaries}\label{sec:prelim}

We consider tabular finite-horizon \emph{Markov decision processes} (MDPs). Formally, an MDP is a tuple $\cM := (\cS, \cA, H, \{p_h\}_{h\in[H]}, \{\nu_h\}_{h\in[H]}, s_1)$, where $\cS$ is a finite set of $S$ states, $\cA$ is a finite set of $A$ actions, $H$ is the horizon, $p_h : \cS\times\cA \rightarrow \cP(\cS)$ and $\nu_h : \cS\times\cA \rightarrow \cP(\bR)$\footnote{$\cP(\cX)$ denotes the set of probability measures over a set $\cX$.} respectively denote the transition kernel and the reward distribution at stage $h\in[H]$, while $s_1\in\cS$ is the initial state\footnote{This setting encompasses any initial state distribution by adding a transition from $s_1$ with the desired probabilities.}. A learner interacts with $\cM$ through episodes of length $H$. At the beginning of each episode, the learner starts in the initial state $s_1$. Then, for each stage $h\in[H]$, the learner plays an action $a_h\in\cA$ and observes a stochastic transition to a new state $s_{h+1} \sim p_h(s_h,a_h)$ as well as a reward $R_h \sim \nu_h(s_h,a_h)$. The actions are usually chosen according to a Markovian (possibly stochastic) policy $\pi = \{\pi_h\}_{h\in[H]}$, i.e., a sequence of mappings $\pi_h : \cS \rightarrow \cP(\cA)$, where $\pi_h(a|s)$ denotes the probability that the learner takes action $a$ in state $s$ at stage $h$. We denote by $\PiS$ (resp. $\PiD$) the set of all Markovian stochastic (resp. deterministic) policies.

\subsection{Policy gaps, value gaps, and state-action distributions} 

Denoting by $\bP^\pi$ (resp. $\bE^\pi$) the probability (resp. expectation) operator induced by the execution of a policy $\pi\in\PiS$ for an episode on $\cM$, we let $V_1^\pi := \bE^\pi\big[\sum_{h=1}^H R_h \big| s_1 \big]$ be the value of $\pi$ at the initial state\footnote{Since the initial state $s_1$ is fixed, we drop it from the notation of value functions.}. The \emph{policy gap} of $\pi$ is then defined as 
\[\Delta(\pi) := V_1^\star - V_1^\pi,\]
%\[\Delta(\pi) := V_1^\star - V_1^\pi, \ \ \text{where} \ \  V_1^\star := \max_{\pi\in\PiD} V_1^\pi \] 
where $V_1^\star := \max_{\pi\in\PiD} V_1^\pi$ is the optimal value at $s_1$. 
We use $Q_h^\pi(s,a) := \bE^\pi\big[\sum_{\ell=h}^H R_\ell \big| s_h = s, a_h = a \big]$ and $Q_h^\star(s,a) := \max_{\pi\in\PiD} Q_h^\pi(s,a)$ to denote the action-value function of $\pi$ and the optimal value function, respectively.
The \emph{value gap} of the triplet $(h,s,a)$ is then defined as \[\Delta_h(s,a) := \max_{b\in \cA} Q_h^\star(s,b) - Q_h^\star(s,a)\;.\] Moreover, we denote the visitation probability of $(h,s,a)$ under $\pi$ as $p_h^\pi(s,a) := \bP^\pi(s_h=s,a_h=a)$ and $p_h^\pi(s) := \bP^\pi(s_h=s)$. We let $\Omega := \big\{\big(p_h^\pi(s,a)\big)_{h,s,a} : \pi\in\PiS\}$ denote the set of all valid state-action distributions. It is well known \citep[e.g.,][]{puterman1994markov} that $\Omega$ is a polytope defined by the linear constraints
\begin{align*}
  \forall \rho \in \Omega,\quad &  \rho_h(s,a) \geq 0\ \forall (h,s,a),\\
  & \sum_{a \in \cA} \rho_1(s_1, a) = 1,\ \sum_{a \in \cA} \rho_1(s, a) = 0\ \forall s\neq s_1, \\
  & \sum_{a \in \cA} \rho_h(s, a) = \sum_{s',a'} \rho_{h-1}(s', a') p_{h-1}(s|s',a')\ \forall (s,h)\in \cS \times[|2,H|].
\end{align*}

\subsection{Learning problem}

 The learner interacts with an MDP $\cM$ with unknown transition probabilities and reward distributions. 
 %in order to return a policy that is provably near-optimal with high probability. More precisely, 
 Given a risk parameter $\delta \in (0,1)$ and a precision $\epsilon \geq 0$ as input, the goal is to return a policy $\widehat{\pi} \in \PiD$ with the guarantee that $\bP_\cM(\Delta(\widehat{\pi})\leq\epsilon) \geq 1-\delta$. To satisfy this requirement, the learner needs to gather samples from the transition and reward distributions of $\cM$ by playing episodes in a sequential fashion. In each episode $t\in\bN^*$, the learner selects a policy $\pi^t$ (based on past observations) and collects a new trajectory $\cH_t := \{(s_h^t,a_h^t, R_h^t)\}_{h\in[H]}$ under this policy, where $a_h^{t} \sim \pi^{t}(s_h^{t})$. We let $\cF_t := \sigma((\cH_u)_{1\leq u \leq t})$ denote the sigma-algebra generated by trajectories up to episode $t$. The learner's performance is then evaluated by its \emph{sample complexity} $\tau$, which is a stopping time w.r.t, the filtration $(\cF_t)_{t\geq 1}$ counting the (random) number of exploration episodes before termination.
\begin{definition}[$(\epsilon,\delta)$-PAC algorithm]\label{def:PAC-alg}
    %Fix  $\epsilon >0$ and $\delta \in (0,1)$. 
    Let $\mathfrak{M}$ be a set of MDPs. An algorithm is $(\epsilon,\delta)$-PAC on $\mathfrak{M}$ %(for Provably Approximately Correct)
     if for all MDPs $\cM \in \mathfrak{M}$, with probability at least $1-\delta$, it stops after playing $\tau < \infty$ episodes on $\cM$ and returns a deterministic policy $\widehat{\pi} \in \PiD$ satisfying 
    $%\label{eq:epsilon-delta-PAC}
      \Delta(\widehat{\pi})\leq\epsilon.
    $
\end{definition}

\section{Lower Bounds}

We consider the class $\mathfrak{M}_1$ of stochastic MDPs with \emph{Gaussian rewards} of unit variance\footnote{We trivially get results for Gaussian rewards with arbitrary variance $\sigma^2$ by multiplying our lower bounds by $\sigma^2$.}, in which $\nu_h(s,a) = \cN(r_h(s,a),1)$. While existing \emph{upper} bounds commonly work under the stronger assumption that rewards lie in $[0,1]$ almost surely, we focus on this alternative setting since it has enabled the derivation of \emph{closed-form lower bounds} that scale with intuitive quantities such as policy gaps \citep{dann21ReturnGap,tirinzoni2022NearIP}. Moreover, the complexity of an MDP is mostly characterized by the expected rewards $r_h(s,a)$ rather than the full distributions $\nu_h(s,a)$, so that matching a lower bound for Gaussians while observing bounded rewards with the same mean is still very informative about the algorithm's adaptivity to the underlying problem.
\subsection{General lower bound for approximate identification}

Our first result is a general bound that holds for any $\epsilon \geq 0$ in the asymptotic regime $\delta \to 0$. We use the notation $\Pi^\epsilon := \{\pi \in \Pi^D : V_1^\pi(s_1) \geq V_1^\star(s_1) - \epsilon\}$ for the set of all deterministic $\epsilon$-optimal policies. 

\begin{theorem}\label{thm:general-lb-epsilon}
Any PAC RL algorithm that is $(\epsilon,\delta)$-PAC on $\mathfrak{M}_1$ satisfies, for any $\cM \in \mathfrak{M}_1$,
\[
      \liminf_{\delta\to 0} \frac{\bE_\cM[\tau]}{\log(1/\delta)} \geq \cC_{\mathrm{LB}}(\cM,\epsilon)\]
      where 
      \[      \cC_{\mathrm{LB}}(\cM,\epsilon):= 2 \min_{\pi^{\epsilon} \in \Pi^{\epsilon}} \min_{\rho\in \Omega} \max_{\pi \in \Pi^D}  \sum_{s,a,h} \frac{\big(p_h^\pi(s,a) - p_h^{\pi^\epsilon}(s,a)\big)^2}{\rho_h(s,a)(\Delta(\pi) -\Delta(\pi^{\epsilon})+\epsilon)^2}.\]

\end{theorem}
Theorem \ref{thm:general-lb-epsilon} states that no matter how adaptive a PAC RL algorithm is, there is a minimal cost in terms of episodes that it must pay in order to learn an $\epsilon$-optimal policy of $\cM$. This cost is instance-dependent since it is a functional of $\cM$, the MDP to be learned. The proof of Theorem \ref{thm:general-lb-epsilon} is  deferred to Appendix \ref{sec:app_lower_bound}. It follows similar steps as the proof of the lower bound for $\epsilon$-best arm identification (and other pure exploration problems) of \cite{DegenneK19}.

\subsection{Finite-$\delta$ bound for exact identification}

In the case of exact identification (i.e. $\epsilon=0$), we further derive a lower bound which is valid for any $\delta \in (0,1)$ under the assumption that the optimal state-action distribution is unique. In particular, we assume that there exists $p^\star \in \Omega$ s.t. for any optimal policy $\pi^\star$ (i.e., with $V_1^{\pi^\star} = V_1^\star$) we have $p^{\pi^\star} = p^\star$. Note that this is a generalization of the assumption of ``unique optimal trajectory'' from \cite{tirinzoni2022NearIP}, under which we know that exact identification to be possible with a sample complexity that does not scale with $\epsilon$. It is also the same assumption considered in \cite{Tirinzoni2021AFP}. As shown in that paper, it implies that there is a unique optimal action in states visited with positive probability by some optimal policy, but there can be arbitrary many optimal actions in all other states.\footnote{Without a unique optimal state-action distribution, exact identification may not be even possible, as no algorithm may be able to stop in finite time and return an optimal policy w.h.p. while being $(0,\delta)$-PAC on the whole family $\mathfrak{M}_1$.}

\begin{theorem}\label{thm:lb-exact-identification}
   Fix any MDP $\cM \in \mathfrak{M}_1$ s.t. the optimal state-action distribution $p^\star$ is unique. Then, for any PAC RL algorithm that is $(0,\delta)$-PAC on $ \mathfrak{M}_1$,
\begin{align*}
      \bE_\cM[\tau] \geq 2\min_{\rho \in \Omega} \max_{\pi\in \PiD : \Delta(\pi) > 0} \sum_{s,a,h} \frac{(p_h^\pi(s,a) - p_h^\star(s,a))^2}{\rho_h(s,a)\Delta(\pi)^2}\log\left(\frac{1}{2.4\delta}\right) .
   \end{align*}
\end{theorem}
\begin{remark}
   When $S =H=1$, this bound exactly coincides with the lower bound for best-arm identification in Gaussian multi-armed bandits \citep{garivier2016optimal}.
\end{remark}
\begin{proof}
   The idea of the proof is to explicitly compute the smallest KL divergence between the distribution of the observations under the MDP $\cM$ and under any alternative $\widetilde{\cM}$ that has the same transitions but a different mean reward function $\widetilde{r}_h$. Within the class $\mathfrak{M}_1$, the KL divergence of observations between $\cM$ and $\widetilde{\cM}$ takes the simple form
   \begin{align*}
      \textrm{KL}(\bP_\cM,\bP_{\widetilde{\cM}}) = \sum_{h,s,a} \bE_\cM[n^\tau_h(s,a)] \frac{\big(r_h(s,a) - \widetilde{r}_h(s,a)\big)^2}{2}.
  \end{align*}
   Note that, since $p^\star$ is unique, any $(0,\delta)$-PAC algorithm satisfies $\bP_{\cM}(V_1^{\widehat\pi} = V_1^\star) = \bP_{\cM}(p^{\widehat{\pi}} = p^\star) \geq 1-\delta$. Now fix a sub-optimal policy $\pi$ for $\cM$ (i.e., with $\Delta(\pi) > 0$). Note that $V_1^\pi = r^T p^\pi < V_1^\star = r^T p^\star$. We look for the closest alternative $\widetilde\cM$ such that $\widetilde{r}^T p^\pi > \widetilde{r}^T p^\star$, i.e., where $\pi$ becomes better than any optimal policy of $\cM$. This can be computed by the quadratic program
   \begin{align*}
      \min_{\widetilde{r} : \widetilde{r}^T p^\pi > \widetilde{r}^T p^\star} \sum_{s,a,h} \bE[n_h^\tau(s,a)] \frac{(r_h(s,a) - \widetilde r_h(s,a))^2}{2} = \frac{\Delta(\pi)^2}{2\sum_{s,a,h} \frac{(p_h^\pi(s,a) - p_h^\star(s,a))^2}{\bE[n_h^\tau(s,a)]}}.
   \end{align*}
   By the $(0,\delta)$-PAC property, in such closest alternative we have $\bP_{\widetilde\cM}(p^{\widehat{\pi}} = p^\star) \leq \delta$. Then, Lemma 1 of \cite{kaufmann2016complexity} ensures that $\textrm{KL}(\bP_\cM,\bP_{\widetilde{\cM}}) \geq \log\left(\frac{1}{2.4\delta}\right)$. Thus, for any $\pi$ with $\Delta(\pi) > 0$,
   % \begin{align*}
   %    \frac{\Delta(\pi)^2}{2\sum_{s,a,h} \frac{(p_h^\pi(s,a) - p_h^\star(s,a))^2}{\bE[n_h^\tau(s,a)]}} \geq \log\left(\frac{1}{2.4\delta}\right).
   % \end{align*}
   \begin{align*}
      1 \geq 2\sum_{s,a,h} \frac{(p_h^\pi(s,a) - p_h^\star(s,a))^2}{\bE[n_h^\tau(s,a)]\Delta(\pi)^2} \log\left(\frac{1}{2.4\delta}\right).
   \end{align*}
  Multiplying both sides by $\bE[\tau]$ and maximizing over sub-optimal policies, we obtain
  \begin{align*}
   \bE[\tau] \geq 2\max_{\pi\in \PiD : \Delta(\pi) > 0}\sum_{s,a,h} \frac{\bE[\tau]}{\bE[n_h^\tau(s,a)]}\frac{(p_h^\pi(s,a) - p_h^\star(s,a))^2}{\Delta(\pi)^2} \log\left(\frac{1}{2.4\delta}\right).
\end{align*}
Now it is easy to see that $\rho_h(s,a) := \bE[n_h^\tau(s,a)]/\bE[\tau]$ is a valid state-action distribution (i.e., $\rho \in \Omega$). Thus, minimizing the right-hand side over all $\rho \in \Omega$ concludes the proof.
\end{proof}

\subsection{Interpreting the lower bound} 

While the expression of the lower bound might seem mysterious at a first glance, we provide an interpretation in terms of confidence intervals for the simpler setting of known transitions and unknown rewards. Our explanation hinges on the following concentration inequality, proved in Appendix~\ref{app:concentration}. 
\begin{lemma}\label{lem:global-concentration-comparison}
   Assume the reward distribution $\nu_h(s,a)$ to be 1-subgaussian\footnote{A random variable $X$ is $\sigma^2$-subgaussian if $\bE[e^{\lambda (X - \bE[ X])}] \leq e^{\sigma^2\lambda^2 / 2}$ for any $\lambda \in \bR$.} with mean $r_h(s,a)$ for all $(h,s,a)$. For any policy $\pi \in \PiD$, define the estimator $\widehat{V}_1^{\pi,t} := \sum_{h,s,a} p_h^{\pi}(s,a) \widehat{r}_h^t(s,a)$, where $\widehat{r}_h^t(s,a)$ is the MLE of $r_h(s,a)$ using samples gathered until episode $t$. We have that 
   \begin{align*}
    \bP\bigg( \forall t\geq t_0,\ \forall \pi,\pi' \in \PiD,\   \big| (\widehat{V}_1^{\pi,t}-\widehat{V}_1^{\pi',t}) - (V_1^\pi - V_1^{\pi'}) \big| \leq \sqrt{\beta(t,\delta) \sum_{h,s,a} \frac{\big(p_h^\pi(s,a)-p_h^{\pi'}(s,a)\big)^2}{n_h^t(s,a)}} \bigg) \geq 1-\delta,
   \end{align*}
  with $t_0 := \inf \{t : n_h^{t}(s,a) \geq 1, \forall (h,s,a)\ \mathrm{s.t.} \sup_\pi p_h^\pi(s) > 0\}$,  and $\beta(t,\delta):= 4\log(1/\delta) + 12 SH\log(A(1+t))$.
  \end{lemma}

Imagine that a learner explores the MDP $\cM$ using a fixed (stochastic) policy $\pi^{\text{exp}}$, whose state-action distribution is $\rho^{\text{exp}}$, and wants to figure out whether some policy $\pi^\epsilon$ is $\epsilon$-optimal or not. Then, after playing $\pi^{\text{exp}}$ for $K\geq 1$ episodes, $\bE[n_h^K(s,a)] = K \rho_h^{\text{exp}}(s,a)$, so that the size of the confidence interval on $V_1^{\pi^\epsilon} - V_1^{\pi}$ should roughly be $\sqrt{\beta(K,\delta) \sum_{h,s,a} \frac{\big(p_h^\pi(s,a) - p_h^{\pi^\epsilon}(s,a)\big)^2}{K \rho_h^{\text{exp}}(s,a)}}$.
Now, if the learner wishes to test whether $\pi^\epsilon$ is $\epsilon$-optimal it has to determine the sign of $V_1^{\pi^\epsilon} - V_1^{\pi}+\epsilon$ for all other policies $\pi$. To do that, it is sufficient to shrink the size of the confidence interval on $V_1^{\pi^\epsilon} - V_1^{\pi}$  below $\frac{1}{2} |V_1^{\pi^\epsilon} - V_1^{\pi}+\epsilon| = \frac{1}{2}|\Delta(\pi) -\Delta(\pi^{\epsilon})+\epsilon|$ for all policies $\pi$. Solving for the minimal $K$ that satisfies this condition, we see that playing roughly
\begin{align*}
   K(\pi^{\text{exp}}, \pi^\epsilon) \propto \log(1/\delta)\max_{\pi \in \Pi^D} \sum_{s,a,h} \frac{\big(p_h^\pi(s,a) - p_h^{\pi^\epsilon}(s,a)\big)^2}{\rho_h^{\text{exp}}(s,a)(\Delta(\pi) -\Delta(\pi^{\epsilon})+\epsilon)^2}
\end{align*} 
episodes using the exploration policy  $\pi^{\text{exp}}$ is enough to determine whether $\pi^\epsilon$ is $\epsilon$-optimal. Since the learner has the liberty to return \textit{any} $\epsilon$-optimal policy using \textit{any} exploration policy, the lower bound corresponds to the minimum of $K(\pi^{\text{exp}}, \pi^\epsilon)$ with respect to these two variables.

\section{Towards a matching upper bound}\label{sec:upper_bounds}

\subsection{Review of existing upper bounds}\label{sec:related_work}  
In this section, we review the main instance-dependent bounds within the PAC RL literature. We restrict our review to works on approximate identification (i.e., the general case with $\epsilon \geq 0$). 
\paragraph{PAC RL with a generative model}
\cite{zanette2019almost} were the first to propose an instance-dependent PAC RL algorithm, called BESPOKE. In infinite-horizon tabular MDPs with a discount factor $\gamma \in [0,1)$ and when the agent has access to a simulator that can query observations from any state-action pair, BESPOKE finds an $\epsilon$-optimal policy with a sample complexity of at most

{\small
\[
   \widetilde{\cO}\bigg(\left[\sum_{s,a} \min\bigg(\frac{1}{(1-\gamma)^3 \epsilon^2},\  \frac{\textrm{Var}[R(s,a)] + \gamma^2 \textrm{Var}_{s'\sim p(.|s,a)}[V^\star(s')]}{\max(\Delta_{sa}, (1-\gamma) \epsilon)^2} + \frac{1}{(1-\gamma)\max(\Delta_{sa}, (1-\gamma) \epsilon)} \bigg)\right]\log\left(\frac{1}{\delta}\right)  \bigg),
\]}
\hspace{-0.2cm} where $\Delta_{sa} = V^\star(s) - Q^\star(s,a)$ is the value gap of state-action pair $(s,a)$ and $\textrm{Var}$ denotes the variance operator. A notable feature of this result is that the sample complexity of BESPOKE (i) scales as $\cO(SA\log(1/\delta)/(1-\gamma)^3\epsilon^2)$ in the worst-case, which is the conjectured minimax lower bound for the infinite-horizon discounted setting \citep{Azar2012MinimaxPB}; (ii) it can be significantly smaller than minimax whenever the MDP is such that playing different actions yields very different total rewards, i.e., when the value gaps $(\Delta_{sa})_{s,a}$ are large compared to $\epsilon$. For the setting of episodic linear MDPs \citep{Jin2019ProvablyER}, the GSS-E algorithm by \cite{Taupin2022BestPI} solves a G-optimal design to determine the sampling frequencies of each state-action pair. The sample complexity of GSS-E is upper bounded by \[\widetilde{\cO}\bigg( \frac{d H^{4}}{(\min_{h,s, a\neq \pi^\star(s)} \Delta_{h}(s,a) +\epsilon)^2} (\log(1/\delta) + d) \bigg),\]
%where $\Delta_{\min}(\cM) = \min_{h,s, a\neq \pi^\star(s)} \Delta_{h}(s,a)$ is the minimum value gap in $\cM$ and 
where $d$ is the feature dimension. Up to $H$ factors, this result improves upon the $\Omega(d^2H^2/\epsilon^2)$ minimax bound for this setting \citep{pmlr-v162-wagenmaker22b} whenever the minimum value gap in $\cM$ is large.

\paragraph{PAC RL without a generative model} On top of the sub-optimality gaps which characterize the bounds above, the instance-dependent complexities feature an additional component when a generative model is not available: visitation probabilities. These constitute the price that PAC RL algorithms have to pay in order to navigate the MDP and collect observations from distant states. Existing high-probability bounds on the sample complexity are of the form\footnote{While we focus on the main complexity terms which scale with sub-optimality gaps, visitation probabilities, and $\log(1/\delta)$, it is worth noting that existing upper bounds all feature lower-order terms in either of these variables.} 
\[\bP\left(\tau = \widetilde{\cO}\left(\cC_{\mathrm{Alg}}(\cM,\varepsilon)\log\left(\frac{1}{\delta}\right)\right)\right) \geq 1-\delta,\]
where $\cC_{\mathrm{Alg}}(\cM,\varepsilon)$ is a complexity measure corresponding to a given algorithm $\mathrm{Alg}$.
For example, for the MOCA algorithm \cite{wagenmaker21IDPAC} obtain
\begin{align*}
  \cC_{\mathrm{MOCA}}(\cM,\varepsilon) = H^2\sum_{h=1}^H \min_{\pi^{\text{exp}}\in\PiS} \max_{s,a} \frac{1}{p_h^{\pi^{\text{exp}}}(s,a)} \min\left[\frac{1}{\Delta_h(s,a)^2},\ \frac{W_h(s)^2}{\epsilon^2}\right] \! + \frac{H^4 \big|\textrm{OPT}(\cM, \epsilon) \big|}{\epsilon^2},
\end{align*}
where $W_h(s) = \sup_{\pi} p_h^\pi(s)$ is the maximum reachability of state $s$ at step $h\in[H]$ and $\textrm{OPT}(\cM, \epsilon)$ is a set of near-optimal triplets $(h,s,a)$. In the above bound, the contribution of a triplet $(h,s,a)$ to the total complexity will be small when either (i) its value gap  $\Delta_h(s,a)$ is large or (ii) it is hard to reach by any policy, that is $W_h(s) \ll \epsilon$. This "local complexity" of $(h,s,a)$ is weighted by $1/p_h^{\pi^{\text{exp}}}(s,a)$, which is the (expected) number of episodes that the agent needs to play in order to reach $(h,s,a)$ when using $\pi^{\text{exp}}$ as an exploration policy. Subsequent works have proposed alternative local complexity measure featuring policy gaps instead of value  gaps \citep{tirinzoni2022NearIP,Wagenmaker22linearMDP, almarjani23covgame}. Policy gaps can be larger than value gaps. Notably, they always are in deterministic MDPs \citep{tirinzoni2022NearIP}. For instance, for the PRINCIPLE algorithm, \cite{almarjani23covgame} obtain
\begin{align*}
   \cC_{\mathrm{PRINCIPLE}}(\cM,\varepsilon) = H^3 \min_{\pi^{\text{exp}}\in\PiS} \max_{h,s,a} \sup_{\pi\in \PiS} \frac{p^\pi_h(s,a)}{p_h^{\pi^{\text{exp}}}(s,a)\max(\epsilon, \Delta(\pi) )^2},
\end{align*}
where we recall the definition of the policy gap $\Delta(\pi) := V_1^\star - V_1^\pi$. Compared to the bound of MOCA, here the contribution of $(h,s,a)$ is small when all policies visiting it are largely sub-optimal. This can be the case even when $a$ is an optimal action in state $s$, provided that no optimal policy reaches $(h,s)$ with positive probability. We note that, while the lower bound of Theorem \ref{thm:general-lb-epsilon} only applies to algorithms that output a deterministic policy (see Definition~\ref{def:PAC-alg}), PRINCIPLE is allowed to return a stochastic policy.
\subsection{PEDEL: A near-optimal algorithm}

The PEDEL algorithm proposed by \cite{Wagenmaker22linearMDP} has the sample complexity bound which resembles the most the complexity measure in our lower bound. To introduce it, we define the minimum policy gap $\Delta_{\min}:= \min_{\pi\in\PiD\setminus\{\pi^\star\}} \Delta(\pi)$, where $\pi^\star$ is an arbitrary optimal policy (i.e., $V_1^{\pi^\star} = V_1^\star$). Note that $\Delta_{\min} = 0$ whenever multiple optimal policies exist. 

While PEDEL tackles the more general setting of identifying a near-optimal policy in linear MDPs, when instantiated for the special case of tabular MDPs, the leading term in its sample complexity bound is
\[\cC_{\mathrm{PEDEL}}(\cM,\varepsilon) = H^4 \sum_{h=1}^H \min_{\rho \in \Omega} \max_{\pi\in\PiD} \sum_{s,a} \frac{p_h^\pi(s,a)^2}{\rho_h(s,a)(\Delta(\pi) \vee \epsilon \vee\Delta_{\min})^2},\]
where we ignore some additive lower-order term that is polynomial in $S,A,H, \log(1/\delta)$ and $\log(1/\epsilon)$. 

The next proposition, proved in Appendix~\ref{app:PEDEL}, compares this complexity measure to our lower bound.

\begin{proposition}\label{prop:PEDEL-lb}
For any MDP $\cM$, it holds that
   \begin{align*}
       \cC_{\mathrm{PEDEL}}(\cM,\epsilon) &\leq 8H^5\cC_{\mathrm{LB}}(\cM,\epsilon) + \frac{4H^6}{(\epsilon\vee\Delta_{\min})^2}.
   \end{align*}

\end{proposition}

This shows that in MDPs in which the minimum policy gap is a constant w.r.t. other problem parameters, i.e., $\Delta_{\min} = \Omega(1)$, the complexity $\cC_{\mathrm{PEDEL}}(\cM,\epsilon)$ is only a factor $H^5$ away from the instance-dependent lower bound. The same conclusion holds when we are interested in the regime $\epsilon = \Omega(1)$. 

More generally, the next proposition provides a sufficient condition on $\cM$ for PEDEL to be instance-optimal up to polynomial multiplicative factors of the horizon, {regardless of the values of $\epsilon$ and $\Delta_{\min}$}. Let us define the following divergence measure between any pair of policies $\pi,\pi'$:
\begin{align*}
  d(\pi,\pi') := \sum_{h\in[H]}\mathrm{TV}(p_h^\pi,p_h^{\pi'})^2,
\end{align*}
where $\mathrm{TV}(p_h^\pi,p_h^{\pi'}) := \frac{1}{2}\sum_{s,a} |p_h^\pi(s,a) - p_h^{\pi'}(s,a)|$ denotes the total variation distance.

\begin{proposition}\label{prop:mdps-pedel-opt}
  Let $\epsilon > 0$ and $\cM$ be an MDP such that, for some constant $c > 0$,
  \begin{align}\label{eq:diversity}
    \min_{\pi^\epsilon\in\Pi^\epsilon} \max_{\pi\in\PiD : \Delta(\pi) \leq \epsilon \vee \Delta_{\min}} d(\pi^\epsilon,\pi) \geq c.
  \end{align}
  Then, $\cC_{\mathrm{PEDEL}}(\cM,\epsilon) \leq 2H^5\big(4 + \frac{H}{c}\big) \cC_{\mathrm{LB}}(\cM,\epsilon)$.
\end{proposition}
Proposition \ref{prop:mdps-pedel-opt} essentially states that, for MDPs where near-optimal policies are sufficiently ``diverse'' (in the sense that for every $\epsilon$-optimal policy there exists a sufficiently distant near-optimal policy), the complexity of PEDEL matches our lower bound up to only multiplicative factors of $H$. There are several classes of MDPs where the ``diversity'' condition \eqref{eq:diversity} is satisfied. For instance, it is sufficient to find two near-optimal policies $\pi^1,\pi^2$ (i.e., such that $\Delta(\pi^1) \vee \Delta(\pi^2) \leq \epsilon \vee \Delta_{\min}$) with $\max_h \mathrm{TV}(p_h^{\pi^1},p_h^{\pi^2}) = 1$ to guarantee that \eqref{eq:diversity} holds with $c=1/4$\footnote{This is because, due to the triangle inequality, $\max\big(\mathrm{TV}(p_h^{\pi^\epsilon},p_h^{\pi^1}), \mathrm{TV}(p_h^{\pi^\epsilon},p_h^{\pi^2})\big) \geq 1/2$ for any $\pi^\epsilon\in\Pi^\epsilon$.}. This happens in either of these cases:
\begin{itemize}
  \item $\pi^1$ and $\pi^2$ deterministically visit some state $s$ at some stage $h$ (i.e., $p_h^{\pi^1}(s)=p_h^{\pi^2}(s)=1$) in which they play different actions (i.e., $\pi_h^1(s) \neq \pi_h^2(s)$). 
  %For instance, it is enough to have two near-optimal actions at the initial state.
  \item $\pi^1$ and $\pi^2$ visit two disjoint sets of states at some stage $h$, i.e., $\{s: p_h^{\pi^1}(s) > 0 \} \cap \{ s: p_h^{\pi^2}(s) > 0\} = \emptyset$.
  \item  $\pi^1$ and $\pi^2$ visit the same states with equal probabilities at some stage $h$ (i.e., $p_h^{\pi^1}(s) = p_h^{\pi^2}(s)$ for any $s$) at which they play different actions (i.e., $\pi_h^1(s) \neq \pi_h^2(s)$ for all $s$). For instance, it is enough to have a constant reward at the last stage (i.e., for some $\alpha$, $r_H(s,a) = \alpha$ for all $s,a$).
%   \item the reward is a constant at the last stage (i.e., for some $\alpha\in\bR$, $r_H(s,a) = \alpha$ for all $s,a$);
\end{itemize}

\smallskip

\begin{remark}
   Upon close inspection of its pseudocode, it seems that PEDEL was designed with the implicit assumption that $\epsilon = \cO(H/d^{3/2})$, where $d$ is the dimension of the linear MDP \footnote{$d = SAH$ in our tabular setting.}. When this assumption is not satisfied (e.g., when $\epsilon = \Omega(1/d)$), the sample complexity of PEDEL can actually be $d$ times larger than $ \cC_{\mathrm{PEDEL}}(\cM,\epsilon)$. We elaborate on this in Appendix \ref{sec:PEDEL-cheats}.
\end{remark}

\section{Conclusion and perspective}

We proposed the  ﬁrst general instance-dependent lower bound for online PAC RL and proved that it is nearly matched by PEDEL \citep{Wagenmaker22linearMDP}. Unfortunately, the algorithm is computationally intractable as it enumerates and stores the set of deterministic policies, which is of size $A^{SH}$, in order to  eliminate suboptimal policies and solve an experimental design of the form
\begin{align}\label{eq:pedel_design}
    \min_{\rho\in\Omega} \max_{\pi\in \Pi_\ell} \sum_{s,a} \frac{\widehat{p}^{\pi, \ell}_h(s,a)^2}{\rho_h(s,a)},
\end{align}
where $\Pi_\ell \subset \PiD$ is the set of active policies at iteration $\ell$ (initialized as $\Pi_0 = \PiD$) and $\widehat{p}^{\pi, \ell}_h(s,a)$ refers to the visitation probabilities of $\pi$ under the empirical MDP $\widehat{\cM}_\ell$.
Therefore, we ask the following question 
\begin{center}
    \emph{Is there a PAC RL algorithm that can (nearly) match our lower bound while requiring a polynomial computational complexity in the size of the MDP?
    %maintaining a computational and memory complexity that are polynomial in $SAH$?
    }
\end{center}
We believe that answering this question would shed light on the (still elusive) problem of instance-optimality in PAC RL. Indeed, if the answer is negative then this would indicate a clear separation between MDPs and bandits, where we know that computationally-efficient instance-optimality is possible \citep{garivier2016optimal, Jedra202OLinearBAI}.

As a starting point to answer the above question, it is natural to wonder whether it is possible to use the same policy-elimination approach as PEDEL while making it computationally efficient. This is precisely the idea of PRINCIPLE \citep{almarjani23covgame}, which performs implicit policy elimination by adding linear constraints to the set of valid state-action distributions. However, while doing so, it only solves an upper bound on the ``optimal'' design (\ref{eq:pedel_design}) used by PEDEL of the form
\begin{align*}
    \min_{\rho\in\Omega} \max_{\eta\in \Omega_\ell} \max_{s,a} \frac{\eta_h(s,a)}{\rho_h(s,a)},
\end{align*}
where $\Omega_\ell$ is the set of valid state-action distributions in $\widehat{\cM}_\ell$ that satisfy certain near-optimality constraints. This makes the sample complexity of the polynomial-time algorithm PRINCIPLE strictly worse than that of PEDEL, thus not matching the lower bound. We leave as an open question whether an implicit policy elimination scheme can be made compatible with the optimal design \eqref{eq:pedel_design}, in which computing the objective itself seems to require enumerating all policies.

\newpage

\bibliography{../biblio_bpi}
\bibliographystyle{plainnat}

%%%%%%%%%%%%%%%%%%%%%%%%%%%%%%%%%%%%%%%%%%%%%%%%%%%%%%%%%%%%%%%%%%%%%%%%%%%%%%%
%%%%%%%%%%%%%%%%%%%%%%%%%%%%%%%%%%%%%%%%%%%%%%%%%%%%%%%%%%%%%%%%%%%%%%%%%%%%%%%
\newpage
\appendix

% Appendix table of contents
\part{Appendix}

% \input{app_aftermainpaperdeadline}

% \clearpage
\parttoc
\newpage

\section{Proof of Theorem \ref{thm:general-lb-epsilon}}\label{sec:app_lower_bound}

As mentioned before, our proof is inspired by the one from \cite{DegenneK19}. The key differences are in Lemma \ref{lem:charateristic-time-simplified} which explicits the shape of the characteristic time for the PAC RL problem and Lemma \ref{lem:likelihood-ratio} which relies on a slightly different martingale construction to concentrate the likelihood ratio. Indeed, our martingale involves the expected number of visits to state-action pairs instead of the actual number of visits as in \cite{DegenneK19}, which is crucial to obtain the navigation constraints $\rho\in \Omega$ in the optimization program of the lower bound. 

\paragraph{Notation}For any $\pi^\epsilon \in \Pi^{\epsilon}$, we define the set of alternative MDPs that have the same transitions as $\cM$ but in which $\pi^\epsilon$ is no longer $\epsilon$-optimal: \[\alt{\pi^\epsilon} := \left\{\widetilde{\cM} \in \mathfrak{M}_1:\ \forall (h,s,a),\ p_h(\cdot|s,a; \widetilde{\cM}) = p_h(\cdot|s,a; \cM)\ \textrm{and}\ \exists \pi\in \Pi^D,\  V_1^{\widetilde{\cM}, \pi^{\epsilon}} < V_1^{\widetilde{\cM}, \pi}-\epsilon \right\}.\] Finally, we define the characteristic time to learn that $\pi^\epsilon$ is $\epsilon$-optimal as
\begin{align*}
    T(\cM, \pi^\epsilon, \epsilon) := \bigg(\sup_{\rho\in \Omega} \inf_{\widetilde{\cM}\in \alt{\pi^\epsilon}} \sum_{h,s,a} \rho_h(s,a) \frac{\big(r_h^{\widetilde{\cM}}(s,a) - r_h^{\cM}(s,a)\big)^2}{2} \bigg)^{-1}.
\end{align*} 
Further, for any set of MDPs $E \subset \mathfrak{M}_1$, we let $\overline{E}$ denote the closure of $E$ where the limit points are defined w.r.t. the distance $d(\cM, \cM') := \max_{h,s,a} |r_h^\cM(s,a) -r_h^{\cM'}(s,a)|$.
\begin{proof}
Let $\xi \in (0,1)$ and define $T := (1-\xi) \min_{\pi^\epsilon \in \Pi^\epsilon}T(\cM, \pi^\epsilon, \epsilon) \log(1/\delta)$\footnote{For simplicity, we assume the latter is an integer.}.  Thanks to Markov's inequality we have that 
\begin{align}\label{ineq:lb-markov}
    \bE_\cM[\tau] \geq T ( 1-\bP_{\cM}(\tau < T)) .
\end{align}
We will now upper bound the probability on the right-hand side above. Since the algorithm is $(\epsilon, \delta)$-PAC We have that 
\begin{align}\label{ineq:proba-union-bound}
 \bP_{\cM}(\tau < T) &=  \bP_{\cM}\big(\widehat{\pi}\notin \Pi^\epsilon, \tau < T\big) + \sum_{\pi^\epsilon \in \Pi^\epsilon} \bP_{\cM}\big(\widehat{\pi}= \pi^\epsilon, \tau <T\big) \nonumber\\
 &\leq \delta + \sum_{\pi^\epsilon \in \Pi^\epsilon} \bP_{\cM}\big(\widehat{\pi}= \pi^\epsilon, \tau <T\big).
\end{align}
Now we fix $\pi^\epsilon \in \Pi^\epsilon$ and apply Lemma \ref{lem:change-of-measure} for the event $\cC = \big(\widehat{\pi}= \pi^\epsilon, \tau <T \big) \in \cF_T$, which yields that there exist $\widetilde{\cM}_1,\ldots, \widetilde{\cM}_{SAH+1} \in \overline{\alt{\pi^\epsilon}}$ and $(\sigma_i)_{1\leq i\leq SAH+1}\in \mathbb{R}_{+}^{SAH+1}$ such that, for all $y>0$,
\begin{align}\label{ineq:invoke-change-of-measure}
    \hspace{-1cm} \bP_{\cM}\big(\widehat{\pi}= \pi^\epsilon, \tau <T\big) &\leq  \exp\big(y+\frac{T}{T(\cM, \pi^\epsilon, \epsilon)}\big) \max_{1\leq i \leq SAH+1} \bP_{\widetilde{\cM}_i}(\widehat{\pi}= \pi^\epsilon, \tau <T\big) \nonumber\\
    & \quad + \sum_{i=1}^{SAH+1} \exp\big(-\frac{y^2}{2 T\sigma_i^2} \big)\nonumber \\
& = \delta^{\xi-1}\exp(y) \max_{1\leq i \leq SAH+1} \bP_{\widetilde{\cM}_i}(\widehat{\pi}= \pi^\epsilon, \tau <T\big) + \sum_{i=1}^{SAH+1} \exp\big(-\frac{y^2}{2 T\sigma_i^2} \big). 
\end{align}
Now for any $i\in [|1, SAH+1|]$ since $\widetilde{\cM}_i \in \overline{\alt{\pi^\epsilon}}$ there exists a sequence of MDPs $(\cM'_n)_{n\geq 1}$ with values in $\alt{\pi^\epsilon}$ such that $ \lim_{n\to \infty} \cM'_n = \widetilde{\cM}_i$\footnote{Recall that the convergence was defined w.r.t. the distance $d(\cM, \cM') := \max_{h,s,a} |r_h^\cM(s,a) -r_h^{\cM'}(s,a)|$}.

By definition of $\alt{\pi^\epsilon}$, we have that $\bP_{\cM'_n}(\widehat{\pi}= \pi^\epsilon, \tau <T\big) \leq \bP_{\cM'_n}(\widehat{\pi}= \pi^\epsilon) \leq \delta$ for all $n\geq 1$. Therefore
\begin{align}\label{ineq:proba-M-tilde}
   \bP_{\widetilde{\cM}_i}(\widehat{\pi}= \pi^\epsilon, \tau <T\big) &\leq \bP_{\widetilde{\cM}_i}(\widehat{\pi}= \pi^\epsilon)\nonumber\\
   &\stackrel{(a)}{\leq} \liminf_{n\to \infty}\bP_{\cM'_n}(\widehat{\pi}= \pi^\epsilon) \leq \delta,
\end{align}
where (a) uses Fatou's lemma.
Combining (\ref{ineq:invoke-change-of-measure}) with (\ref{ineq:proba-M-tilde}) for the value $y = \xi\log(1/\delta)/2$ yields
\begin{align}\label{ineq:individual-proba-pi-epsilon}
 \bP_{\cM}\big(\widehat{\pi}= \pi^\epsilon, \tau <T\big)&\leq \delta^{\xi}\exp(y)+ \sum_{i=1}^{SAH+1} \exp\bigg(-\frac{y^2}{2 T\sigma_i^2} \bigg)\nonumber\\
 &\stackrel{(a)}{=} \delta^{\xi/2}+ \sum_{i=1}^{SAH+1} \exp\bigg(- \frac{\xi^2\log(1/\delta)}{4 (1-\xi) \min_{\pi^\epsilon \in \Pi^\epsilon}T(\cM, \pi^\epsilon, \epsilon) \sigma_i^2} \bigg),
\end{align}
where (a) uses the definition of $T$. Therefore $\lim_{\delta\to 0}  \bP_{\cM}\big(\widehat{\pi}= \pi^\epsilon, \tau <T\big) = 0$. This, combined with (\ref{ineq:proba-union-bound}) gives that $ \lim_{\delta\to 0} \bP_{\cM}(\tau < T) = 0$. Plugging this back into (\ref{ineq:lb-markov}) and using the definition of $T$ yields
$$\liminf_{\delta\to 0} \frac{\bE_\cM[\tau]}{\log(1/\delta)} \geq (1-\xi) \min_{\pi^\epsilon \in \Pi^\epsilon}T(\cM, \pi^\epsilon, \epsilon).$$
To finish the proof of Theorem \ref{thm:general-lb-epsilon}, we take the limit when $\xi$ goes to zero and use the simplified expression of the characteristic time given in Lemma \ref{lem:charateristic-time-simplified}.
\end{proof}

\subsection{The change-of-measure argument}
\begin{lemma}\label{lem:change-of-measure}
Consider $(\widetilde{\cM}_i)_{1\leq i \leq SAH+1} \in \overline{\alt{\pi^\epsilon}}^{SAH+1}$ given by Lemma \ref{lem:T-max-min-game} and let $T\geq 1$. Then for any event $\cC \in \cF_T$ and any $y > 0$ we have

\begin{align*}
   \bP_\cM(C) \leq \exp\big(y+\frac{T}{T(\cM, \pi^\epsilon, \epsilon)}\big) \max_{1\leq i \leq SAH+1} \bP_{\widetilde{\cM}_i}(C) + \sum_{i=1}^{SAH+1} \exp\big(- \frac{y^2}{2 T\sigma_i^2} \big),
\end{align*}
where $\sigma_i^2 := \frac{H^2}{4}d(\cM,\widetilde{\cM}_i)^2(1+d(\cM,\widetilde{\cM}_i))^2$.
\end{lemma}
\begin{proof}
Consider the simplex vector $\lambda^\star \in \Delta_{SAH+1}$ given by Lemma \ref{lem:T-max-min-game}. We define the mixture distribution $ \bQ = \sum_{i=1}^{SAH+1} \lambda_i^\star \bP_{\widetilde{\cM}_i}$ and the corresponding log-likelihood ratio 
$$
  L_T(\bP_\cM, \bQ) := \ln\frac{d\bP_\cM}{d\bQ}(\cH_T) . 
$$
Using Lemma 3.1 from \citep{garivier2021nonasymptotic} we have that for any event $C \in \cF_T$ and any $x>0$, 
\begin{align}\label{ineq:low-level-change-of-measure}
  \hspace{1.5cm} \bP_\cM(C) \leq e^x \bQ(C) + \bP_\cM(L_T(\bP_\cM, \bQ) > x). 
\end{align}
We bound each term in the right-hand side separately. Since $\lambda^\star \in \Delta_{SAH+1}$, for any event $C$,
\begin{align}\label{ineq:change-measure-bound-max-prob}
    \hspace{1.5cm} \bQ(C) &= \sum_{i=1}^{SAH+1} \lambda_i^\star \bP_{\widetilde{\cM}_i}(C) \leq \max_{1\leq i \leq SAH+1} \bP_{\widetilde{\cM}_i}(C)
\end{align}
On the other hand, we have that 
\begin{align*}
 L_T(\bP_\cM, \bQ) &\stackrel{(a)}{\leq} \sum_{i=1}^{SAH+1} \lambda_i^\star \ln\frac{d\bP_\cM}{d\bP_{\widetilde{\cM}_i}}\bigg((s_1^t, a_1^t, R_1^t, \ldots, s_H^t, a_H^t, R_H^t)_{1\leq t \leq T} \bigg)\\
 &= \sum_{i=1}^{SAH+1} \lambda_i^\star L_T(\bP_\cM,\bP_{\widetilde{\cM}_i})\\
 &\stackrel{(b)}{=} \sum_{i=1}^{SAH+1} \lambda_i^\star M_T(\bP_\cM,\bP_{\widetilde{\cM}_i})+  \sum_{i=1}^{SAH+1} \lambda_i^\star \sum_{h,s,a} \bE_\cM[n^T_h(s,a)] \frac{\big(r_h^{\widetilde{\cM}_i}(s,a) - r_h^{\cM}(s,a)\big)^2}{2}\\
 &= \sum_{i=1}^{SAH+1} \lambda_i^\star M_T(\bP_\cM,\bP_{\widetilde{\cM}_i})+ T \sum_{i=1}^{SAH+1} \lambda_i^\star \sum_{h,s,a} \frac{\bE_\cM[n^T_h(s,a)]}{T} \frac{\big(r_h^{\widetilde{\cM}_i}(s,a) - r_h^{\cM}(s,a)\big)^2}{2}\\
 &\stackrel{(c)}{\leq} \sum_{i=1}^{SAH+1} \lambda_i^\star M_T(\bP_\cM,\bP_{\widetilde{\cM}_i}) + \frac{T}{T(\cM, \pi^\epsilon, \epsilon)},
\end{align*}

where (a) uses the convexity of $x \mapsto \log(1/x)$ and Jensen's inequality, (b)
uses Lemma \ref{lem:likelihood-ratio} and (c) uses the second statement of Lemma \ref{lem:T-max-min-game} and the fact that the vector $\big[\frac{\bE_\cM[n^T_h(s,a)]}{T}\big]_{h,s,a}$ belongs to $\Omega(\cM)$. Therefore for any $y>0$, we have that
\begin{align}\label{ineq:log-likelihood-azuma}
\bP_\cM\bigg(L_T(\bP_\cM, \bQ) > \frac{T}{T(\cM, \pi^\epsilon, \epsilon)}+ y \bigg) &\leq \bP_\cM\bigg(\sum_{i=1}^{SAH+1} \lambda_i^\star M_T(\bP_\cM,\bP_{\widetilde{\cM}_i}) > y\bigg)\nonumber \\
&\leq \sum_{i=1}^{SAH+1} \bP_\cM\bigg(M_T(\bP_\cM,\bP_{\widetilde{\cM}_i}) > y \bigg)\nonumber \\
&\leq \sum_{i=1}^{SAH+1} \exp\bigg(- \frac{y^2}{2 T \sigma_i^2} \bigg),
\end{align}

where in the last line we defined $\sigma_i^2 := \frac{H^2}{4}d(\cM,\widetilde{\cM}_i)^2(1+d(\cM,\widetilde{\cM}_i))^2$ and used Azuma-Hoeffding inequality along with Lemma \ref{lem:likelihood-ratio}. Combining (\ref{ineq:change-measure-bound-max-prob}) and (\ref{ineq:log-likelihood-azuma}) with (\ref{ineq:low-level-change-of-measure}) for $x = \frac{T}{T(\cM, \pi^\epsilon, \epsilon)}+ y$ gives the result.
\end{proof}

\subsection{A max-min game formulation}
We define $\Delta_{SAH+1} := \{\lambda\in\mathbb{R}_{+}^{SAH+1}:\ \sum_{i=1}^{SAH+1} \lambda_i = 1 \}$ to be the simplex of dimension $SAH$. Further, for any set of MDPs $E \subset \mathfrak{M}_1$, we let $\overline{E}$ denote the closure of $E$ where the convergence is defined w.r.t. the distance $d(\cM, \cM') := \max_{h,s,a} |r_h^\cM(s,a) -r_h^{\cM'}(s,a)|$. $\mathrm{Conv}(E)$ refers to the convex hull of $E$. Finally, we define the set of KL-divergence vectors generated by alternative instances in $\alt{\pi^\epsilon}$, 
\begin{align*}
\cD(\pi^\epsilon) := \bigg\{\bigg[\frac{\big(r_h^{\widetilde{\cM}}(s,a) - r_h^{\cM}(s,a)\big)^2}{2}\bigg]_{h,s,a}\in \mathbb{R}^{SAH}\ \textrm{s.t.}\ \widetilde{\cM}\in \alt{\pi^\epsilon} \bigg\}. 
\end{align*}
\begin{lemma}\label{lem:T-max-min-game}
Fix $\pi^\epsilon \in \Pi^{\epsilon}$.   

Then there exists $\rho^\star \in \Omega, \lambda^\star \in \Delta_{SAH+1}$ and $\widetilde{\cM}_1, \ldots, \widetilde{\cM}_{SAH+1} \in \overline{\alt{\pi^\epsilon}}$ such that 
\begin{align*}
  T(\cM, \pi^\epsilon, \epsilon)^{-1} = \sum_{i=1}^{SAH+1} \lambda_i^\star\bigg[\sum_{h,s,a} \rho_h^\star(s,a) \frac{\big(r_h^{{\widetilde{\cM}_i}}(s,a) - r_h^{\cM}(s,a)\big)^2}{2} \bigg].
\end{align*}
Furthermore, for any $\rho \in \Omega$ we have that 
\begin{align*}
    \sum_{i=1}^{SAH+1} \lambda_i^\star\bigg[\sum_{h,s,a} \rho_h (s,a) \frac{\big(r_h^{{\widetilde{\cM}_i}}(s,a) - r_h^{\cM}(s,a)\big)^2}{2} \bigg] &\leq T(\cM, \pi^\epsilon, \epsilon)^{-1}.
\end{align*}
\end{lemma}
\begin{proof}
Observe that we can rewrite the expression of the characteristic time $T(\cM, \pi^\epsilon, \epsilon)$ as follows,
\begin{align}\label{eq:characteristic-time-maxmin}
  T(\cM, \pi^\epsilon, \epsilon)^{-1} &= \sup_{\rho\in \Omega} \inf_{\widetilde{\cM}\in \alt{\pi^\epsilon}} \sum_{h,s,a} \rho_h(s,a) \frac{\big(r_h^{\widetilde{\cM}}(s,a) - r_h^{\cM}(s,a)\big)^2}{2} \nonumber\\
  &= \sup_{\rho\in \Omega} \inf_{\widetilde{d}\in \cD(\pi^\epsilon)} \rho^{\top}\widetilde{d} \nonumber\\
  &= \sup_{\rho\in \Omega} \inf_{\widetilde{d}\in \overline{\cD(\pi^\epsilon)}} \rho^{\top}\widetilde{d} \nonumber\\
  &= \sup_{\rho\in \Omega} \inf_{\widetilde{d}\in \mathrm{Conv}(\overline{\cD(\pi^\epsilon)})} \rho^{\top}\widetilde{d},
\end{align}
where $\mathrm{Conv}(\cD(\pi^\epsilon))$ denotes the convex hull of $\cD(\pi^\epsilon)$. Now let $(\rho^\star, d^\star)$ be an optimal solution to (\ref{eq:characteristic-time-maxmin}). Since $\cD(\pi^\epsilon)\subset \mathbb{R}^{SAH}$, by Carathéodory's extension theorem we have that there exists $\lambda^\star \in \Delta_{SAH+1}$ and $d_1,\ldots, d_{SAH+1} \in \overline{\cD(\pi^\epsilon)}$ such that $d^\star = \sum_{i=1}^{SAH+1} \lambda_i^\star d_i$. This means that there exists $\rho^\star \in \Omega$ and $\widetilde{\cM}_1, \ldots, \widetilde{\cM}_{SAH+1} \in \overline{\alt{\pi^\epsilon}}$ such that 
\begin{align*}
  T(\cM, \pi^\epsilon, \epsilon)^{-1} &= (\rho^\star)^\top d^\star\\
  &= \sum_{i=1}^{SAH+1} \lambda_i^\star (\rho^\star)^\top d_i\\
  &= \sum_{i=1}^{SAH+1} \lambda_i^\star\bigg[\sum_{h,s,a} \rho_h^\star(s,a) \frac{\big(r_h^{{\widetilde{\cM}_i}}(s,a) - r_h^{\cM}(s,a)\big)^2}{2} \bigg]. 
\end{align*}
This proves the first statement. Now for the second statement, using Sion's minimax theorem (\cite{Sion1958OnGM}, Theorem 3.4) we know that
\begin{align*}
    (\rho^\star)^\top d^\star = \sup_{\rho\in \Omega} \inf_{\widetilde{d}\in \mathrm{Conv}(\overline{\cD(\pi^\epsilon)})} \rho^{\top}\widetilde{d} = \inf_{\widetilde{d}\in \mathrm{Conv}(\overline{\cD(\pi^\epsilon)})} \sup_{\rho\in \Omega} \rho^{\top}\widetilde{d},
\end{align*}
i.e $(\rho^\star, d^\star)$ is a saddle point of  (\ref{eq:characteristic-time-maxmin}). This means that for all $\rho \in \Omega$
\begin{align*}
    \rho^\top d^\star \leq (\rho^\star)^\top d^\star = T(\cM, \pi^\epsilon, \epsilon)^{-1}.
\end{align*}
Expanding the left-hand side proves the second statement.
\end{proof}

\subsection{Log-likelihood ratio for MDPs with the same transition kernel}

In the following we fix an algorithm $\mathfrak{A}$. For $T\geq 1$ we define the history up to the end of episode $T$ as $\cH_T := (s_1^t, a_1^t, R_1^t, \ldots, s_H^t, a_H^t, R_H^t, \indi{t \leq \tau_{\delta}})_{1\leq t \leq T}$. For any MDP $\cM$, we write $\bP_{\cM}$ to denote the probability distribution over possible histories when $\mathfrak{A}$ interacts with $\cM$\footnote{Since we will be considering the same algorithm $\mathfrak{A}$ interacting with different MDPs, we do not index the probability distributions by $\mathfrak{A}$.}. Further $(\cF_T)_{T \geq 1}$ will denote the sigma algebra generated by $(\cH_T)_{T\geq 1}$. Finally, for a pair of MDPs $\cM, \widetilde{\cM}$, we define the log-likelihood ratio of observations at the end of any episode $T$\footnote{With the convention that $p_0(\cdot|s_0,a_0) = \mathds{1}(s_1=\cdot)$ for all $(s_0,a_0)$. Also note that we have simplified the probabilities of choosing actions $\pi^t(a_h^t|s_h^t,a_{h-1}^t,\ldots, s_1^t, \cH_{t-1)}$ and of stopping $\pi^t(\tau_\delta = t|\cH_t)$ as they only depend on the history, therefore having the same value for $\cM$ and $\widetilde{\cM}$.}
\begin{align*}
    L_T(\bP_\cM,\bP_{\widetilde{\cM}}) &:= \ln\frac{d\bP_\cM}{d\bP_{\widetilde{\cM}}}(\cH_T )\\
    &= \ln\bigg( \prod_{t=1}^T\prod_{h=1}^H \frac{\exp\big(-[R_h^t- r_h^\cM(s_h^t,a_h^t)]^2/2\big) p_{h-1}^{\cM}(s_h^t| s_{h-1}^t, a_{h-1}^t)}{\exp\big(-[R_h^t- r_h^{\widetilde{\cM}}(s_h^t,a_h^t)]^2/2\big) p_{h-1}^{\widetilde{\cM}}(s_h^t| s_{h-1}^t, a_{h-1}^t)}\bigg).
\end{align*}
\begin{lemma}\label{lem:likelihood-ratio}
For any pair of MDPs $\cM, \widetilde{\cM} \in \mathfrak{M}_1$, there exists a martingale (under $\bE_{\cM}$) $\big(M_T(\bP_\cM,\bP_{\widetilde{\cM}})\big)_{T\geq 1}$ whose increments are $\frac{H^2}{4}d(\cM,\widetilde{\cM})^2(1+d(\cM,\widetilde{\cM}))^2$-subgaussian and such that the likelihood ratio at the end of episode $T$ satisfies 
\begin{align*}
    L_T(\bP_\cM,\bP_{\widetilde{\cM}}) = M_T(\bP_\cM,\bP_{\widetilde{\cM}})+  \sum_{h,s,a} \bE_\cM[n^T_h(s,a)] \frac{\big(r_h^{\widetilde{\cM}}(s,a) - r_h^{\cM}(s,a)\big)^2}{2}.
\end{align*}
\end{lemma}

\begin{proof}
Using that the MDPs $\cM$ and $\widetilde{\cM}$ share the same transition kernels and have Gausssian reward distributions with unit variance, we can simplify their log-likelihood ratio as follows,
\begin{align}\label{eq:log-likelihood-1}
    L_T(\bP_\cM,\bP_{\widetilde{\cM}}) &= -\frac{1}{2} \sum_{t=1}^T \sum_{h=1}^H \bigg[\big(R_h^t - r_h^{\cM}(s_h^t,a_h^t)\big)^2 - \big(R_h^t - r_h^{\widetilde{\cM}}(s_h^t,a_h^t)\big)^2 \bigg]\nonumber\\
    &= \frac{1}{2} \sum_{h,s,a} \sum_{t=1}^T \mathds{1}(s_h^t = s, a_h^t = a) \bigg[\big(R_h^t - r_h^{\widetilde{\cM}}(s,a)\big)^2 - \big(R_h^t - r_h^{\cM}(s,a)\big)^2\bigg].
\end{align}
Now for any fixed $(h,s,a)$ we can define $\widehat{r}^T_h(s,a) := \frac{\sum_{t=1}^T \mathds{1}(s_h^t = s, a_h^t = a) R_h^t}{n^T_h(s,a)}$ if $n^T_h(s,a) > 0$ and $\widehat{r}^T_h(s,a):=0$ otherwise. Then we can write that 
\begin{align}\label{eq:log-likelihood-2}
  &\sum_{t=1}^T \mathds{1}(s_h^t = s, a_h^t = a)\big(R_h^t - r_h^{\cM}(s_h,a_h)\big)^2\nonumber\\
    &= \sum_{t=1}^T \mathds{1}(s_h^t = s, a_h^t = a) \bigg[\big(R_h^t - \widehat{r}^T_h(s,a)\big) + \big(\widehat{r}^T_h(s,a)- r_h^{\cM}(s,a)\big) \bigg]^2\nonumber\\
  &= \sum_{t=1}^T \mathds{1}(s_h^t = s, a_h^t = a) \bigg[ \big(R_h^t - \widehat{r}^T_h(s,a)\big)^2+ \big(\widehat{r}^T_h(s,a)- r_h^{\cM}(s,a)\big)^2 \bigg] \nonumber\\
  &+ 2 \big(\widehat{r}^T_h(s,a)- r_h^{\cM}(s,a)\big) \underbrace{\sum_{t=1}^T \mathds{1}(s_h^t = s, a_h^t = a) \big(R_h^t - \widehat{r}^T_h(s,a)\big)}_{=0}\nonumber\\
  &= \sum_{t=1}^T \mathds{1}(s_h^t = s, a_h^t = a) \bigg[ \big(R_h^t - \widehat{r}^T_h(s,a)\big)^2+ \big(\widehat{r}^T_h(s,a)- r_h^{\cM}(s,a)\big)^2 \bigg].
\end{align}
Similarly, one can show that
\begin{align}\label{eq:log-likelihood-3}
    &\sum_{h,s,a} \sum_{t=1}^T \mathds{1}(s_h^t = s, a_h^t = a) \big(R_h^t - r_h^{\widetilde{\cM}}(s_h^t,a_h^t)\big)^2 \nonumber\\
    &= \sum_{t=1}^T \mathds{1}(s_h^t = s, a_h^t = a) \bigg[ \big(R_h^t - \widehat{r}^T_h(s,a)\big)^2+ \big(\widehat{r}^T_h(s,a)- r_h^{\widetilde{\cM}}(s,a)\big)^2 \bigg].
\end{align}
Combining equations (\ref{eq:log-likelihood-1}), (\ref{eq:log-likelihood-2}) and (\ref{eq:log-likelihood-3}) we get that
\begin{align}\label{eq:log-likelihood-4}
   L_T(\bP_\cM,\bP_{\widetilde{\cM}}) &= \frac{1}{2} \sum_{h,s,a} \sum_{t=1}^T \mathds{1}(s_h^t = s, a_h^t = a) \bigg[ \big(\widehat{r}^T_h(s,a)- r_h^{\widetilde{\cM}}(s,a)\big)^2-  \big(\widehat{r}^T_h(s,a)- r_h^{\cM}(s,a)\big)^2 \bigg]\nonumber\\
   &= \frac{1}{2} \sum_{h,s,a} n^T_h(s,a) \bigg(r_h^{\cM}(s,a)- r_h^{\widetilde{\cM}}(s,a)\bigg) \bigg(2\widehat{r}^T_h(s,a) -r_h^{\cM}(s,a) - r_h^{\widetilde{\cM}}(s,a)\bigg).
\end{align}
Next we define the sequences
\begin{align*}
    M_T(h,s,a) &:= \frac{1}{2}\bigg[n^T_h(s,a) \big(r_h^{\cM}(s,a)- r_h^{\widetilde{\cM}}(s,a)\big) \big(2\widehat{r}^T_h(s,a) -r_h^{\cM}(s,a) - r_h^{\widetilde{\cM}}(s,a)\big)\\
    &- \bE_\cM[n^T_h(s,a)]\big(r_h^{\widetilde{\cM}}(s,a) - r_h^{\cM}(s,a)\big)^2 \bigg].\\
    M_T(\bP_\cM,\bP_{\widetilde{\cM}}) &:= \sum_{h,s,a} M_T(h,s,a).
\end{align*}
Using (\ref{eq:log-likelihood-4}) one can check that
\begin{align*}
   L_T(\bP_\cM,\bP_{\widetilde{\cM}}) = M_T(\bP_\cM,\bP_{\widetilde{\cM}})+  \sum_{h,s,a} \bE_\cM[n^T_h(s,a)] \frac{\big(r_h^{\widetilde{\cM}}(s,a) - r_h^{\cM}(s,a)\big)^2}{2}. 
\end{align*}
This proves the second statement. Now for the first statement we note that for $T\geq 2$,
\begin{align*}
    &M_T(\bP_\cM,\bP_{\widetilde{\cM}}) - M_{T-1}(\bP_\cM,\bP_{\widetilde{\cM}}) \\ &= \frac{1}{2} \sum_{h,s,a}\bigg(r_h^{{\cM}}(s,a) - r_h^{\widetilde{\cM}}(s,a)\bigg)\mathds{1}(s_h^T = s, a_h^T = a) \bigg(2R_h^T - r_h^{\cM}(s,a) - r_h^{\widetilde{\cM}}(s,a) \bigg)\\
    &\quad\quad -  \bP_{\cM}(s_h^T = s, a_h^T = a)\bigg(r_h^{\cM}(s,a) - r_h^{\cM}(s,a)\bigg)^2 \\
    &= \underbrace{\frac{1}{2}\sum_{h,s,a}\left(r_h^{\cM}(s,a) - r_h^{\widetilde{\cM}}(s,a)\right)\ind\left(s_h^{T}=s,a_h^{T}=a\right)\left(R_h^{T} - r_h^{\cM}(s,a)\right)}_{:= X_T} \\ & \quad+ \underbrace{\frac{1}{2}\sum_{h,s,a}\left(r_h^{\cM}(s,a) - r_h^{\widetilde{\cM}}(s,a)\right)^2\left(\ind\left(s_h^{T}=s,a_h^{T}=a\right) - \bP_{\cM}\left(s_h^{t}=s,a_h^{t}=a\right)\right).}_{:= Y_T}
\end{align*}
$X_T$ satisfies 
\[\bE[X_T|\cF_{T-1}] = \bE\left[\left.\frac{1}{2}\sum_{h,s,a} \left(r_h^{\cM}(s,a) - r_h^{\widetilde{\cM}}(s,a)\right) \underbrace{\bE\left[\left(R_h^{T} - r_h^{\cM}(s,a)\right) | S_h^{T},A_h^{T}\right]}_{= 0} \right|\cF_{T-1}\right] = 0\]
and $X_T = \sum_{h=1}^{H}X_{T,h}$ where \[X_{T,h} = \frac{\left(r_h^{\cM}(s_h^{T},a_h^{T}) - r_h^{\widetilde{\cM}}(s_h^{T},a_h^{T})\right)}{2}\left(R_h^{T} - r_h^{\cM}(s_h^{T},a_h^{T})\right)\] is subgaussian with variance $\frac{d(\cM,\widetilde{\cM})^2}{4}$  conditionally to $\cF_{T-1}$ (using that $R_h^{T} - r_h^{\cM}(s_h^{T},a_h^{T})$ is $1$-subgaussian). Therefore, by Lemma~\ref{lem:sum-of-subgaussian} stated below, $X_T$ is subgaussian with $\sigma_X^2 = \frac{H^2d(\cM,\widetilde{\cM})^2}{4}$.

$Y_T$ satisfies 
\[\bE[Y_T|\cF_{T-1}] = \frac{1}{2}\sum_{h,s,a} \left(r_h^{\cM}(s,a) - r_h^{\widetilde{\cM}}(s,a)\right)^2 \bE\left[\ind(s_h^{T} = s, a_h^{T} = a) - \bP^{\cM}(s_h^{T} = s, a_h^{T} = a) |\cF_{T-1}\right] = 0\]
and $|Y_T| \leq \frac{Hd(\cM,\widetilde{\cM})^2}{2}$. Therefore $Y_T$ is subgaussian with $\sigma_Y^{2} = \frac{H^2d(\cM,\widetilde{\cM})^4}{4}$. 

By Lemma~\ref{lem:sum-of-subgaussian}, $M_T(\bP_\cM,\bP_{\widetilde{\cM}}) - M_{T-1}(\bP_\cM,\bP_{\widetilde{\cM}})$ is conditionally subgaussian with variance \[\frac{H^2d(\cM,\widetilde{\cM})^2(1+d(\cM,\widetilde{\cM}))^2}{4}.\]

\end{proof}

\begin{lemma}[sum of subgaussian random variables, e.g. \cite{Buldygin80SG}]\label{lem:sum-of-subgaussian}
Let $X$ an $Y$ be two random variables that are $\sigma_X^2$ and $\sigma_Y^2$ subgaussian respectively. Then  $X+Y$ is $(\sigma_X+\sigma_Y)^2$-subgaussian.
\end{lemma}
\begin{proof}
Using Hölder inequality and the definition of subgaussian variables, we can write, for any $p\geq1, q\geq 1$ such that $\frac{1}{p}+\frac{1}{q}=1$
\begin{align*}
    \bE[\exp\big(t(X+Y)\big)] &= \bE[\exp(tX) \exp(tY)]\\
    &\leq \bE[\exp(ptX)]^{1/p} \bE[\exp(qtY)]^{1/q}\\
    &{\leq} \exp\left(\frac{p^2t^2\sigma_X^2}{2}\right)^{1/p} \exp\left(\frac{q^2t^2\sigma_Y^2}{2}\right)^{1/q}\\
    &= \exp\left(\frac{t^2(p\sigma_X^2 + q\sigma_Y^2)}{2}\right).
\end{align*}
The conclusion follows by choosing $p = \frac{\sigma_X+\sigma_Y}{\sigma_X}$ and $q = \frac{\sigma_X+\sigma_Y}{\sigma_Y}$ for which $p\sigma_X^2 + q\sigma_Y^2 = (\sigma_X + \sigma_Y)^2$.

\end{proof}

\subsection{Simplifying the expression of the characteristic time}
\begin{lemma}\label{lem:charateristic-time-simplified}
For any $\cM \in \mathfrak{M}_1$ and $\pi^\epsilon \in \Pi^\epsilon$ we have
\begin{align*}
    T(\cM, \pi^\epsilon, \epsilon) = 2 \inf_{\rho\in \Omega} \max_{\pi \in \Pi^D}  \sum_{s,a,h} \frac{\big(p_h^\pi(s,a) - p_h^{\pi^\epsilon}(s,a)\big)^2}{\rho_h(s,a)(\Delta(\pi) -\Delta(\pi^{\epsilon})+\epsilon)^2}.
\end{align*}
\end{lemma}
\begin{proof}
Let us first solve the inner minimization program in the definition of $T(\cM, \pi^\epsilon, \epsilon)^{-1}$. Using the definition of $\alt{\pi^\epsilon}$, we have that {\small
\begin{align}\label{eq:inf_problem}
    \inf_{\widetilde{\cM}\in \alt{\pi^\epsilon}} \sum_{h,s,a} \rho_h(s,a) \tfrac{\big(r_h^{\widetilde{\cM}}(s,a) - r_h^{\cM}(s,a)\big)^2}{2}  = \min_{\pi\in\PiD} \inf_{\widetilde{\cM}: V_1^{\widetilde{\cM}, \pi^{\epsilon}}\!\! < V_1^{\widetilde{\cM}, \pi}\!-\epsilon} \ \sum_{h,s,a} \rho_h(s,a) \tfrac{\big(r_h^{\widetilde{\cM}}(s,a) - r_h^{\cM}(s,a)\big)^2}{2}\;.
\end{align}}
Now observe that we can rewrite $V_1^{\widetilde{\cM}, \pi^{\epsilon}} < V_1^{\widetilde{\cM}, \pi}-\epsilon$ as linear constraint in the rewards of $\widetilde{\cM}$:
\begin{align}
    &\sum_{h,s,a} (p_h^\pi(s,a) - p_h^{\pi^\epsilon}(s,a)) r_h^{\widetilde{\cM}}(s,a) > \epsilon, \nonumber\\
    &\Longleftrightarrow \sum_{h,s,a} (p_h^\pi(s,a) - p_h^{\pi^\epsilon}(s,a)) \big(r_h^{\widetilde{\cM}}(s,a) - r_h^{\cM}(s,a)\big)> V_1^{\pi^{\epsilon}} - V_1^{\pi} +\epsilon, \nonumber\\
    & \Longleftrightarrow \sum_{h,s,a} (p_h^\pi(s,a) - p_h^{\pi^\epsilon}(s,a)) \big(r_h^{\widetilde{\cM}}(s,a) - r_h^{\cM}(s,a)\big)> \Delta(\pi) -\Delta(\pi^{\epsilon})+\epsilon\nonumber
\end{align}
Therefore, letting $u_h(s,a) = r_h^{\widetilde{\cM}}(s,a) - r_h^{\cM}(s,a)$, the program in (\ref{eq:inf_problem}) is equivalent to
\begin{align}
    \min_{\pi\in \PiD} \inf_{\substack{u \textrm { s.t}:\\ \\  
    \sum\limits_{h,s,a} (p_h^\pi(s,a) - p_h^{\pi^\epsilon}(s,a)) u_h(s,a) >  \Delta(\pi) -\Delta(\pi^{\epsilon})+\epsilon}} \sum_{h,s,a} \rho_h(s,a) \frac{u_h(s,a)^2}{2}.
\end{align}
Solving the KKT conditions of the previous program, we get that
\[
    \inf_{\substack{u \textrm { s.t}:\\ \\  
    \sum\limits_{h,s,a} (p_h^\pi(s,a) - p_h^{\pi^\epsilon}(s,a)) u_h(s,a) > \Delta(\pi) -\Delta(\pi^{\epsilon})+\epsilon}} \sum_{h,s,a} \rho_h(s,a) \frac{u_h(s,a)^2}{2} = \bigg(\sum_{h,s,a} \frac{(p_h^\pi(s,a) - p_h^{\pi^\epsilon}(s,a))^2}{\rho_h(s,a) (\Delta(\pi) -\Delta(\pi^{\epsilon})+\epsilon)^2}\bigg)^{-1}.
\]
Summing up all the inequalities, we conclude that
\begin{align*}
    T(\cM, \pi^\epsilon, \epsilon)^{-1} = \frac{1}{2} \sup_{\rho\in\Omega} \min_{\pi\in \PiD} \bigg(\sum_{h,s,a} \frac{(p_h^\pi(s,a) - p_h^{\pi^\epsilon}(s,a))^2}{\rho_h(s,a) (\Delta(\pi) -\Delta(\pi^{\epsilon})+\epsilon)^2}\bigg)^{-1}.
\end{align*}
\end{proof}

\section{Concentration results}\label{app:concentration}

We report here useful concentration results from previous literature.
\begin{proposition}(\textsc{Lemma 26, \cite{almarjani23covgame}}\footnote{Note that, while \cite{almarjani23covgame} state this lemma for rewards bounded in $[0,1]$, they actually prove it for any 1-subgaussian distribution. Indeed, their proof simply combines the concentration result of \cite{abbasi2011improved}, which holds for any subgaussian distribution, with a trick from \cite{reda2021dealing}.})\label{prop:conc-r-term}
    Let the reward distribution $\nu_h(s,a)$ be 1-subgaussian with mean $r_h(s,a)$ for all $(h,s,a)$, and let $\widehat{r}_h^t(s,a)$ be the MLE of $r_h(s,a)$ using samples gathered until episode $t$. Let $\cZ \subseteq [H] \times \cS \times \cA$ and $Z := |\cZ|$. With probability at least $1-\delta$, for any $t \geq t_0 := \inf \{t : n_h^{t}(s,a) \geq 1, \forall (h,s,a)\in\cZ\}$,
    \begin{align*}
        \sum_{(h,s,a)\in\cZ} n_h^t(s,a) \big(\widehat{r}_h^t(s,a) - r_h(s,a)\big)^2 \leq 4\log(1/\delta) + 2Z \log(1+t).
    \end{align*}
\end{proposition}

\begin{proposition}(\textsc{Lemma 30, \cite{almarjani23covgame}})\label{prop:solve-program}
    Let $n\in\mathbb{N}$, $q,b\in\bR^n$ with $b$ having strictly positive entries, and $c \in \bR_{\geq 0}$. Then,
    \begin{align*}
        \sup_{\substack{x \in \bR^{n} :\\
        \sum_{i=1}^n b_i x_i^2 \leq c}} \sum_{i=1}^n q_i x_i = \sqrt{c\sum_{i=1}^n \frac{q_i^2}{b_i}} .
    \end{align*}
\end{proposition}

\subsection{Proof of Lemma \ref{lem:global-concentration-comparison}}

\begin{proof}
Fix any pair of policies $\pi, \pi'$. We write
\begin{align*}
    (\widehat{V}_1^{\pi,t}-\widehat{V}_1^{\pi',t}) - (V_1^\pi - V_1^{\pi'}) &= (p^\pi - p^{\pi'})^\top (\widehat{r}^t - r)\\
    &= \sum_{h,s,a} (p^\pi_h(s,a) - p^{\pi'}_h(s,a)) (\widehat{r}^t_h(s,a) - r_h(s,a)) \\
    &= \sum_{h,s,a} \mathds{1}\big(a\in \{\pi_h(s), \pi_h'(s)\}\big)(p^\pi_h(s,a) - p^{\pi'}_h(s,a)) (\widehat{r}^t_h(s,a) - r_h(s,a)),  
\end{align*}
where we used vector notation $p^\pi = [p_h^\pi(s,a)]_{h,s,a}$. Now applying Proposition \ref{prop:conc-r-term} with $\delta' = \delta/(A^{2SH})$ and the set $\cZ = \big\{(h, s, a)\ |\ (h,s)\in [H]\times\cS, a\in \{\pi_h(s), \pi_h'(s)\}\ \mathrm{s.t. } \sup_\pi p_h^\pi(s) > 0  \big\}$ we get that with probability at least $1- \delta'$,
\begin{align*}
    \forall t \geq t_0,\ \sum_{h,s, a}  \mathds{1}\big(a\in \{\pi_h(s), \pi_h'(s)\}\big) n_h^t(s, a) \big(\widehat{r}_h^t(s,a) - r_h(s,a)\big)^2  &\leq 4\log(1/\delta') + 4 SH\log(A(1+t))\\
    &:= \widetilde{\beta}(t,\delta'),
\end{align*}
where we used that $|\cZ| \leq 2SH$. Next, for each pair of policies $(\pi,\pi')$ we use Proposition \ref{prop:solve-program} with $q = p^\pi - p^{\pi'}$ which yields that
\begin{align*}
    \big| (\widehat{V}_1^{\pi,t}-\widehat{V}_1^{\pi',t}) - (V_1^\pi - V_1^{\pi'}) \big| &\leq \sqrt{\widetilde{\beta}(t,\delta') \sum_{h,s,a}\mathds{1}\big(a\in \{\pi_h(s), \pi_h'(s)\}\big) \frac{\big(p_h^\pi(s,a)-p_h^{\pi'}(s,a)\big)^2}{n_h^t(s,a)}}\\
    &= \sqrt{\widetilde{\beta}(t,\delta') \sum_{h,s,a} \frac{\big(p_h^\pi(s,a)-p_h^{\pi'}(s,a)\big)^2}{n_h^t(s,a)}},  
\end{align*}
with probability at least $1- \delta/(A^{2SH})$. We conclude the proof with a union bound over pairs of policies $(\pi, \pi') \in \PiD\times \PiD$ and remarking that 
\[\widetilde{\beta}(t,\delta') =  4 \log(1/\delta) + 12SH\log(A) + 4SH\log(1+t) \leq \beta(t,\delta).\]
\end{proof}

\section{PEDEL} \label{app:PEDEL}

\subsection{Proof of Proposition~\ref{prop:PEDEL-lb}}

First, let us introduce the intermediate complexity measure
\begin{align*}
C(\cM,\epsilon) :=  \min_{\rho \in \Omega} \max_{\pi\in\PiD} \sum_{s,a,h} \frac{p_h^\pi(s,a)^2}{\rho_h(s,a)(\Delta(\pi) \vee \epsilon \vee\Delta_{\min})^2}.
\end{align*}
We start by showing that $H^3 C(\cM,\epsilon) \leq \cC_{\mathrm{PEDEL}}(\cM,\epsilon) \leq H^5 C(\cM,\epsilon)$. For $h\in [H]$ consider any $\rho^{\star,h} \in \argmin_{\rho \in \Omega} \max_{\pi\in\PiD} \sum_{s,a} \frac{p_h^\pi(s,a)^2}{\rho_h(s,a)(\Delta(\pi) \vee \epsilon \vee\Delta_{\min})^2}$. Now, letting $\widetilde{\rho} := \frac{1}{H}\sum_{h=1}^H \rho^{\star,h}$, we see that since $\Omega$ is a convex set, $\widetilde{\rho} \in \Omega$. Furthermore, 
\begin{align*}
C(\cM,\epsilon) &=  \min_{\rho \in \Omega} \max_{\pi\in\PiD} \sum_{s,a,h} \frac{p_h^\pi(s,a)^2}{\rho_h(s,a)(\Delta(\pi) \vee \epsilon \vee\Delta_{\min})^2}\\
&\leq  \max_{\pi\in\PiD} \sum_{s,a,h} \frac{p_h^\pi(s,a)^2}{\widetilde{\rho}_h(s,a)(\Delta(\pi) \vee \epsilon \vee\Delta_{\min})^2}\\
&\stackrel{(a)}{\leq} \sum_{h=1}^H \max_{\pi\in\PiD} \sum_{s,a} \frac{p_h^\pi(s,a)^2}{\widetilde{\rho}_h(s,a)(\Delta(\pi) \vee \epsilon \vee\Delta_{\min})^2}\\
&\stackrel{(b)}{\leq} H \sum_{h=1}^H \max_{\pi\in\PiD} \sum_{s,a} \frac{p_h^\pi(s,a)^2}{\rho_h^{\star,h}(s,a)(\Delta(\pi) \vee \epsilon \vee\Delta_{\min})^2}\\
&= H \sum_{h=1}^H   \min_{\rho \in \Omega} \max_{\pi\in\PiD} \sum_{s,a} \frac{p_h^\pi(s,a)^2}{\rho_h(s,a)(\Delta(\pi) \vee \epsilon \vee\Delta_{\min})^2}\\
&= H^{-3}  \cC_{\mathrm{PEDEL}}(\cM,\epsilon),
\end{align*}
where (a) uses the fact that $\max_{\pi} \sum_{h} f(\pi, h) \leq \sum_{h}  \max_{\pi} f(\pi, h)$ and (b) uses the crude bound $\widetilde{\rho}_h(s,a) \geq \rho^{\star,h}_h(s,a)/H$. On the other hand we have
\begin{align*}
\cC_{\mathrm{PEDEL}}(\cM,\epsilon) &= H^4 \sum_{h=1}^H \min_{\rho \in \Omega} \max_{\pi\in\PiD} \sum_{s,a} \frac{p_h^\pi(s,a)^2}{\rho_h(s,a)(\Delta(\pi) \vee \epsilon \vee\Delta_{\min})^2}\\
&\leq H^4 \sum_{\ell=1}^H \min_{\rho \in \Omega} \max_{\pi\in\PiD} \sum_{s,a,h} \frac{p_h^\pi(s,a)^2}{\rho_h(s,a)(\Delta(\pi) \vee \epsilon \vee\Delta_{\min})^2}\\
&= H^5 C(\cM,\epsilon).
\end{align*}
Therefore, we just proved that
\begin{align}\label{ineq:PEDEL-intermediate}
H^3 C(\cM,\epsilon) \leq \cC_{\mathrm{PEDEL}}(\cM,\epsilon) \leq H^5 C(\cM,\epsilon).
\end{align}
Now we compare $C(\cM,\epsilon)$ and $\cC_{\mathrm{LB}}(\cM,\epsilon)$. Using that $a^2 \leq 2(a-b)^2 + 2b^2$, we note that for any $\rho\in\Omega$ and any $\pi^\epsilon \in \Pi^\epsilon$,
\begin{align}\label{ineq:complexity-decomposition-1}
\max_{\pi\in\PiD} & \frac{\sum_{s,a,h} \frac{p_h^\pi(s,a)^2}{\rho_h(s,a)}}{(\Delta(\pi) \vee \epsilon \vee\Delta_{\min})^2}\nonumber
\\ &\leq  \max_{\pi\in\PiD} \bigg[\sum_{s,a,h} \frac{2(p_h^\pi(s,a) - p_h^{\pi^\epsilon}(s,a))^2}{\rho_h(s,a)(\Delta(\pi) \vee \epsilon \vee\Delta_{\min})^2} + \sum_{s,a,h} \frac{2p_h^{\pi^\epsilon}(s,a)^2}{\rho_h(s,a)(\Delta(\pi) \vee \epsilon \vee\Delta_{\min})^2} \bigg]\nonumber\\
&\leq \max_{\pi\in\PiD} \sum_{s,a,h} \frac{2(p_h^\pi(s,a) - p_h^{\pi^\epsilon}(s,a))^2}{\rho_h(s,a)(\Delta(\pi) \vee \epsilon \vee\Delta_{\min})^2}+ \max_{\pi\in\PiD}\sum_{s,a,h} \frac{2p_h^{\pi^\epsilon}(s,a)^2}{\rho_h(s,a)(\Delta(\pi) \vee \epsilon \vee\Delta_{\min})^2}\nonumber\\
&= \max_{\pi\in\PiD}  \sum_{s,a,h} \frac{2(p_h^\pi(s,a) - p_h^{\pi^\epsilon}(s,a))^2}{\rho_h(s,a)(\Delta(\pi) \vee \epsilon \vee\Delta_{\min})^2} + \sum_{s,a,h} \frac{ 2p_h^{\pi^\epsilon}(s,a)^2}{\rho_h(s,a)(\epsilon \vee\Delta_{\min})^2}.
\end{align}
Now let us define $\rho^{0} := \argmin_{\rho \in \Omega} \max_{\pi\in\PiD} \sum_{s,a,h} \frac{(p_h^\pi(s,a) - p_h^{\pi^\epsilon}(s,a))^2}{\rho_h(s,a)(\Delta(\pi) \vee \epsilon \vee\Delta_{\min})^2}$ and $\widetilde{\rho}^1 := \frac{\rho^{0}+ p^{\pi^\epsilon}}{2}$. Then we have that
\begin{align*}
C(\cM,\epsilon) &=  \min_{\rho \in \Omega} \max_{\pi\in\PiD} \sum_{s,a,h} \frac{p_h^\pi(s,a)^2}{\rho_h(s,a)(\Delta(\pi) \vee \epsilon \vee\Delta_{\min})^2}\\
&\leq \max_{\pi\in\PiD} \sum_{s,a,h} \frac{p_h^\pi(s,a)^2}{\widetilde{\rho}_h^1(s,a)(\Delta(\pi) \vee \epsilon \vee\Delta_{\min})^2}\\
&\stackrel{(a)}{\leq} \max_{\pi\in\PiD} \sum_{s,a,h} \frac{2(p_h^\pi(s,a) - p_h^{\pi^\epsilon}(s,a))^2}{\widetilde{\rho}_h^1(s,a)(\Delta(\pi) \vee \epsilon \vee\Delta_{\min})^2} + \sum_{s,a,h} \frac{ 2p_h^{\pi^\epsilon}(s,a)^2}{\widetilde{\rho}_h^1(s,a)(\epsilon \vee\Delta_{\min})^2}\\
&\stackrel{(b)}{\leq} \max_{\pi\in\PiD} \sum_{s,a,h} \frac{4(p_h^\pi(s,a) - p_h^{\pi^\epsilon}(s,a))^2}{\rho_h^0(s,a)(\Delta(\pi) \vee \epsilon \vee\Delta_{\min})^2} + \sum_{s,a,h} \frac{ 4p_h^{\pi^\epsilon}(s,a)^2}{p_h^{\pi^\epsilon}(s,a)(\epsilon \vee\Delta_{\min})^2}\\
&= 4\min_{\rho \in \Omega} \max_{\pi\in\PiD} \sum_{s,a,h} \frac{(p_h^\pi(s,a) - p_h^{\pi^\epsilon}(s,a))^2}{\rho_h(s,a)(\Delta(\pi) \vee \epsilon \vee\Delta_{\min})^2} + \frac{4H}{(\epsilon\vee\Delta_{\min})^2},
\end{align*}
where (a) uses (\ref{ineq:complexity-decomposition-1}) and (b) uses the fact that for all $(h,s,a)$, $\widetilde{\rho}_h^1(s,a) \geq \max(\rho_h^0(s,a), p_h^{\pi^\epsilon}(s,a))/2$. Since this holds for any $\pi^\epsilon$,
%Now taking the minimum over $\pi^\epsilon \in \Pi^\epsilon$ in both sides of the previous inequality proves that
\begin{align}\label{ineq:intermediate-lb-1}
C(\cM,\epsilon) &\leq 4\min_{\pi\in\Pi^\epsilon} \min_{\rho \in \Omega} \max_{\pi\in\PiD} \sum_{s,a,h} \frac{(p_h^\pi(s,a) - p_h^{\pi^\epsilon}(s,a))^2}{\rho_h(s,a)(\Delta(\pi) \vee \epsilon \vee\Delta_{\min})^2} + \frac{4H}{(\epsilon\vee\Delta_{\min})^2}\nonumber\\
&\leq 16 \min_{\pi\in\Pi^\epsilon} \min_{\rho \in \Omega} \max_{\pi\in\PiD} \sum_{s,a,h} \frac{(p_h^\pi(s,a) - p_h^{\pi^\epsilon}(s,a))^2}{\rho_h(s,a)(\Delta(\pi) + \epsilon - \Delta(\pi^\epsilon))^2} + \frac{4H}{(\epsilon\vee\Delta_{\min})^2}\nonumber\\
&= 8\cC_{\mathrm{LB}}(\cM,\epsilon) + \frac{4H}{(\epsilon\vee\Delta_{\min})^2},  
\end{align} 
where in the second inequality we used that $\Delta(\pi) + \epsilon - \Delta(\pi^\epsilon) \leq 2(\Delta(\pi) \vee \epsilon \vee\Delta_{\min})$. Combining (\ref{ineq:PEDEL-intermediate}) and (\ref{ineq:intermediate-lb-1}) proves the first inequality.

\subsection{Proof of Proposition \ref{prop:mdps-pedel-opt}}\label{app:pedel-opt}

Combining the first inequality in the sequence \eqref{ineq:intermediate-lb-1} with \eqref{ineq:PEDEL-intermediate}, we have that
  \begin{align}\label{eq:pedel-decomposition}
    \cC_{\mathrm{PEDEL}}(\cM,\epsilon) \leq 4H^5\underbrace{\min_{\pi\in\Pi^\epsilon} \min_{\rho \in \Omega} \max_{\pi\in\PiD} \sum_{s,a,h} \frac{(p_h^\pi(s,a) - p_h^{\pi^\epsilon}(s,a))^2}{\rho_h(s,a)(\Delta(\pi) \vee \epsilon \vee\Delta_{\min})^2}}_{(\star)} + \frac{4H^6}{(\epsilon\vee\Delta_{\min})^2}.
  \end{align}
  We now lower bound $(\star)$ as a function of $1/(\epsilon\vee\Delta_{\min})^2$. We have
  \begin{align*}
    (\star) &\geq \min_{\pi\in\Pi^\epsilon} \min_{\rho \in \Omega} \max_{\pi\in\PiD : \Delta(\pi) \leq \epsilon \vee \Delta_{\min}} \sum_{s,a,h} \frac{(p_h^\pi(s,a) - p_h^{\pi^\epsilon}(s,a))^2}{\rho_h(s,a)(\epsilon \vee\Delta_{\min})^2}
    \\ &\geq \min_{\pi\in\Pi^\epsilon} \max_{\pi\in\PiD : \Delta(\pi) \leq \epsilon \vee \Delta_{\min}} \sum_{h\in[H]} \min_{\rho \in \Omega} \sum_{s,a} \frac{(p_h^\pi(s,a) - p_h^{\pi^\epsilon}(s,a))^2}{\rho_h(s,a)(\epsilon \vee\Delta_{\min})^2}
    \\ &\geq \min_{\pi\in\Pi^\epsilon} \max_{\pi\in\PiD : \Delta(\pi) \leq \epsilon \vee \Delta_{\min}} \sum_{h\in[H]} \min_{\rho \in \cP(\cS\times\cA)} \sum_{s,a} \frac{(p_h^\pi(s,a) - p_h^{\pi^\epsilon}(s,a))^2}{\rho_h(s,a)(\epsilon \vee\Delta_{\min})^2}
    \\ &= \frac{1}{(\epsilon \vee\Delta_{\min})^2}\min_{\pi\in\Pi^\epsilon} \max_{\pi\in\PiD : \Delta(\pi) \leq \epsilon \vee \Delta_{\min}} \sum_{h\in[H]} \left(\sum_{s,a} |p_h^\pi(s,a) - p_h^{\pi^\epsilon}(s,a)|\right)^2
    \\ &= \frac{4}{(\epsilon \vee\Delta_{\min})^2}\min_{\pi\in\Pi^\epsilon} \max_{\pi\in\PiD : \Delta(\pi) \leq \epsilon \vee \Delta_{\min}} d(\pi^\epsilon,\pi) \geq \frac{4c}{(\epsilon \vee\Delta_{\min})^2},
  \end{align*}
  where the first equality uses that $\min_{\rho\in\cP(\cX)} \sum_{x\in\cX} \frac{f(x)}{\rho(x)} = (\sum_{x\in\cX}\sqrt{f(x)})^2$ for any non-negative function $f$. This implies that
  \begin{align*}
    \frac{4H^6}{(\epsilon\vee\Delta_{\min})^2} \leq \frac{H^6}{c}(\star).
  \end{align*}
  Plugging this into \eqref{eq:pedel-decomposition} and using that $(\star) \leq 2\cC_{\mathrm{LB}}(\cM,\epsilon)$ as in \eqref{ineq:intermediate-lb-1} concludes the proof.

\subsection{On the complexity of PEDEL in the moderate $\epsilon$ regime}\label{sec:PEDEL-cheats}

PEDEL has a loop structure where at each iteration it seeks to halve the precision of its estimate of the value for all the policies that are still active. Taking a closer look into the design of PEDEL, we notice that it starts the first iteration with the parameter $\ell_0 = \lceil \log_2 \frac{d^{3/2}}{H} \rceil$ and ends at $\lceil \log \frac{4}{\epsilon} \rceil$. From Theorem 7 in \cite{Wagenmaker22linearMDP}, we get that the number of episodes played during the initial iteration is 
\begin{align*}
  &\cO\bigg(H^4 \sum_{h=1}^H \frac{\inf_{\Lambda_{exp}\in \Omega_h} \max_{\phi\in\Phi} \norm{\phi}_{\Lambda_{exp}^{-1}}}{\epsilon_{exp}}\bigg),\ \textrm{where}\quad \epsilon_{exp} := \frac{\epsilon_{\ell_0}^2}{\beta_{\ell_0}},\\
  &\epsilon_{\ell_0} := 2^{-\ell_0} = \frac{H}{d^{3/2}},\ \beta_{\ell_0} := 64H^2 \log(\frac{4H^2 |\Pi|\ell_0^2}{\delta}).  
\end{align*}
As a consequence, running just the initial iteration of PEDEL requires the number of episodes 
\begin{align*}
    \cC_0 := \cO\bigg(d^3 H^4\log(|\Pi|/\delta) \min_{\rho \in \Omega} \max_{\pi\in\PiD} \sum_{s,a,h} \frac{p_h^\pi(s,a)^2}{\rho_h(s,a)} \bigg)\;.
\end{align*}

When $\epsilon = \Omega(1/d)$, we have that $d^2 =\Omega( \frac{1}{(\epsilon\vee\Delta(\pi)\vee \Delta_{\min})^2})$ for all policies $\pi$ so that
\begin{align*}
    \cC_0 = \Omega\bigg(d H^4 \log(|\Pi|/\delta) \min_{\rho \in \Omega} \max_{\pi\in\PiD} \frac{\sum_{s,a,h} \frac{p_h^\pi(s,a)^2}{\rho_h(s,a)}}{(\epsilon\vee\Delta(\pi)\vee \Delta_{\min})^2} \bigg).
\end{align*}

Therefore when $\epsilon = \Omega(1/SAH)$, we get that the sample complexity of PEDEL for tabular MDPs satisfies
\[\tau = \Omega\left(SAH \times \cC_{\mathrm{PEDEL}}(\cM,\epsilon)\log\left(1/\delta\right)\right),\]
almost surely.

%%%%%%%%%%%%%%%%%%%%%%%%%%%%%%%%%%%%%%%%%%%%%%%%%%%%%%%%%%%%%%%%%%%%%%%%%%%%%%%
%%%%%%%%%%%%%%%%%%%%%%%%%%%%%%%%%%%%%%%%%%%%%%%%%%%%%%%%%%%%%%%%%%%%%%%%%%%%%%%
\end{document}